\documentclass[11pt,twocolumn]{article}

\usepackage[margin=0.75in]{geometry}
\usepackage{times}
\usepackage{graphicx}
\usepackage{amsthm,amsmath,amssymb}
\usepackage{xcolor}
\usepackage{booktabs,array,enumitem}
\usepackage[round,authoryear]{natbib}
\usepackage{url}
\usepackage{hyperref}

\title{\bfseries Scaling Up Bayesian DAG Sampling}

\makeatletter
\renewcommand{\@maketitle}{
  \newpage
  \null
  \vskip 0.5em
  \begin{center}
    {\LARGE \@title \par}
    \vskip 2.5em 
    \begin{tabular}{ccc}
      \textbf{Daniele Nikzad} & \textbf{Alexander Zhilkin} & \textbf{Juha Harviainen} \\
      \textit{University of Helsinki} & \textit{University of Helsinki} & \textit{University of Helsinki} \\
      {\footnotesize\texttt{daniele.nikzad@helsinki.fi}} &
      {\footnotesize\texttt{alexander.zhilkin@helsinki.fi}} &
      {\footnotesize\texttt{juha.harviainen@helsinki.fi}} \\
      \\[-0.2em]
      \textbf{Jack Kuipers} & \textbf{Giusi Moffa} & \textbf{Mikko Koivisto} \\
      \textit{ETH Zurich} & \textit{University of Basel} & \textit{University of Helsinki} \\
      {\footnotesize\texttt{jack.kuipers@bsse.ethz.ch}} &
      {\footnotesize\texttt{giusi.moffa@unibas.ch}} &
      {\footnotesize\texttt{mikko.koivisto@helsinki.fi}}
    \end{tabular}
    \vskip 2.5em
  \end{center}
}
\makeatother

\date{}

\renewenvironment{abstract}{
  \begin{center}
    \Large\bfseries Abstract 
  \end{center}
  \vspace{-0.5em}
  \small
}{\par\vspace{1em}}

\long\def\comment#1{}

\newtheorem{theorem}{Theorem}

\newtheorem{proposition}[theorem]{Proposition}

\theoremstyle{definition}
\newtheorem{definition}{Definition}

\theoremstyle{remark}
\newtheorem{remark}{Remark}

\newcommand{\be}{\begin{eqnarray}}
\newcommand{\ee}{\end{eqnarray}}
\newcommand{\bes}{\begin{eqnarray*}}
\newcommand{\ees}{\end{eqnarray*}}

\newcommand{\Zh}[1]{Z_{\textrm{head}, #1}}
\newcommand{\Zt}[1]{Z_{\textrm{tail}, #1}}

\newcommand{\asia}{{\sc Asia}}
\newcommand{\sachs}{{\sc Sachs}}
\newcommand{\child}{{\sc Child}}

\newcommand{\alarm}{{\sc Alarm}}

\newcommand{\hailfinder}{{\sc Hailfinder}}

\newcommand{\pathfinder}{{\sc Pathfinder}}

\newcommand{\andes}{{\sc Andes}}

\newcommand{\pigs}{{\sc Pigs}}

\newcommand{\zoo}{{\sc Zoo}}

\newcommand{\gibby}{\emph{Gibby}}

\newcommand{\ourbasic}{\emph{Gibby}}
\newcommand{\GC}{\emph{GC}}
\newcommand{\bidag}{\emph{BiDAG}}
\newcommand{\daggflow}{\emph{DAG-GFlowNet}}

\newcommand{\TimeB}{$T_{\textrm{b}}$}
\newcommand{\TimeC}{$T_{\textrm{c}}$}
\newcommand{\cpp}{C++}

\begin{document}

\twocolumn[
\maketitle

\begin{abstract}
Bayesian inference of Bayesian network structures is often performed by sampling directed acyclic graphs along an appropriately constructed Markov chain. We present two techniques to improve sampling. First, we give an efficient implementation of basic moves, which add, delete, or reverse a single arc. Second, we expedite summing over parent sets, an expensive task required for more sophisticated moves: we devise a preprocessing method to prune possible parent sets so as to approximately preserve the sums. Our empirical study shows that our techniques can yield substantial efficiency gains compared to previous methods.
\end{abstract}
\vspace{1em}
]

\section{INTRODUCTION}

Directed acyclic graphs (DAGs) provide a concise way to express conditional independence relations among random variables, decomposing a multivariate distribution into a product of univariate conditional distributions \citep{Pearl1988}. Such models, known as Bayesian networks or belief networks (BNs), can also be naturally attached with causal semantics, with an arc from one node to another in the graph corresponding to a (possible) direct cause--effect relation between the associated variables \citep{Pearl2009}.  

Numerous methods have been proposed to learn DAGs from data, usually modelling the data points as independent draws from an unknown BN, but otherwise following various learning paradigms (\citealp{Kitson2023}; \citealp[Ch.~18]{Koller2009}; \citealp{Rios2021}). The natural approach to account for the uncertainty is the Bayesian one, which formalizes learning as the transformation of a prior distribution over DAGs into a posterior \citep{Buntine1991,Cooper1992,Heckerman1995}. Within this paradigm, different methods are mainly ranked by their \emph{computational efficiency}, that is, how fast they can (reliably) produce a good approximation of the posterior, exact algorithms being feasible for networks on only about twenty or fewer variables   \citep{Harviainen2024,Koivisto2004,Pensar2020,Talvitie2019uai,Tian2009}.

Approximating a DAG posterior also serves as a concrete challenge for general-purpose methodology for probabilistic inference. The most prominent approaches taken so far are Markov chain Monte Carlo \citep[MCMC; see, e.g.,][]{Friedman2003,Giudici2003,Grzegorczyk2008,Kuipers2022,Madigan1995,Su2016}, 
variational inference \citep{Annadani2023,Cundy2021}, and generative flows \citep{Deleu2022,Deleu2023}; while MCMC 
is based on sampling that converges to the posterior, the latter two aim at learning a function that approximates the posterior. Ideas that prove useful for dealing with the hard-to-handle space of DAGs, therefore, have the potential to be valuable also more generally.

This paper advances \emph{structure MCMC}, a class of methods we review in Section~\ref{se:prel}. The new ideas we present not only improve DAG sampling, but may also find applications in other contexts.

We begin with a basic Markov chain that moves by adding, removing, or reversing a single arc.  \citet{Giudici2003} gave an efficient way to simulate such a chain, which we refer to as the \GC{} algorithm. It maintains the ancestor relation of the DAG to enable quick acyclicity checks. In Section~\ref{se:gibby}, we significantly expedite \GC{}, by \emph{one to three orders of magnitude} in our experiments. Our key idea is to use an appropriate upper bound for the (typically small) acceptance probability of the proposed move, so that the time staying in the current state can be simulated simply by a geometric random variable.

In Section~\ref{se:pruning}, we consider improved Markov chains that, besides basic moves, also resample the entire parent set for certain nodes \citep{Goudie2016,Grzegorczyk2008,Su2016}. While such moves can make bigger changes to the DAG, they also require computing sums of local scores over large numbers of possible parent sets. The same challenge is encountered with MCMC methods that primarily sample node orderings instead of DAGs \citep{Friedman2003,Kuipers2017,Kuipers2022,Niinimaki2013,Viinikka2020uai,Viinikka2020}. Previous works have addressed this challenge by limiting the number of parents, by only allowing parents from a few preselected candidates, and by accumulating the sums in decreasing order of scores until a reasonable approximation is achieved  \citep{Friedman2003,Kuipers2022,Niinimaki2013,Viinikka2020}. Here, we add a new tool: a pruning method that discards parent sets whose local scores can be proven not to contribute significantly to any relevant score sum. Our experiments show significant savings, by \emph{one to three orders of magnitude}.

In Section~\ref{se:comparison}, we integrate the new techniques into an enhanced implementation of structure MCMC we dub \gibby{} (available at \url{https://github.com/sums-of-products/gibby}) and study its performance against other available Bayesian DAG samplers. Our interest is in \emph{computational complexity}: how reliably and fast the samplers can produce a good approximation of the well-defined DAG posterior. We stress that, in this work, our primary interest is not in \emph{modeling and learning accuracy}: neither in finding a single DAG that ``best fits the data,'' nor in learning arcs or other features of the data-generating graph (even when one exists). 
Throughout the paper, we illustrate our findings using selected benchmark BNs \citep{ScutariBNrepository}. 

\section{PRELIMINARIES}\label{se:prel}

We begin by recalling the basics of Bayesian learning of BNs, and then review the previous development of structure MCMC, on which we build in later sections.  

\subsection{Bayesian Learning of DAGs}

Consider a vector $x = (x_1, \ldots, x_n)$ of $n$ random variables. 
A BN represents a joint probability distribution of $x$ as a pair $(G, f)$, where $G = (N, A)$ is a directed acyclic graph (DAG) on the node set $N = \{1,\ldots,n\}$ and arc set $A$, and $f$ factorizes into a product of univariate conditional distributions along the graph: $f(x) = \prod_{i \in N} f(x_i | x_{A_i})$.
Here $A_i = \{ j : ji \in A \}$ is the parent set of $i$ in $G$, and for any set $S$ we write $x_S$ for the tuple $(x_j : j \in S)$. In what follows, the node set~$N$ is typically fixed, whereas the arc set $A$ varies. We may thus identify a DAG $G$ with its arc set $A$.

To learn a BN $(G, f)$ from a data set $X$ of $m$ points 
$x^{(1)}, \ldots, x^{(m)}$, we model them as independent draws from $f$. Taking a Bayesian approach, we build a joint model by multiplying the likelihood $P(X | G, f) = \prod_{t=1}^m f(x^{(t)})$ by a prior $P(G, f) = P(G)P(f|G)$. The posterior $P(G, f | X)$ is the outcome of learning. 

In the experiments, we adopt the following forms, unless stated otherwise. 
For the structure prior, we use $P(G) \propto c^{-|A|}$, with either $c = 1$ or $c = n$; the former is uniform over all DAGs and thus concentrates almost all mass on dense graphs; the latter, we call \emph{sparse}, renders large parent sets  unlikely.\footnote{For a node with $k$ parents, the term in the sparse prior is $n^{-k} \approx \binom{n-1}{k}^{-1}/k!$. With the term $\binom{n-1}{k}^{-1}$ alone, the prior is roughly uniform over parent set sizes \citep{Eggeling2019,Friedman2003}. The mean of the distribution $p_k \propto 1/k!$ on $k \in \mathbb{N}$ is $1$. More generally, by replacing $n$ by $n/\mu$ in the sparse prior, the expected indegree per node becomes roughly $\mu$.} 
For the prior on $f$, we assume each $x_i$ is categorical, parameterize $f$ by full conditional probability tables, and assign the parameters independent Dirichlet priors. Then the marginal likelihood has a closed-form expression, the Bayesian Dirichlet (BD) score \citep{Buntine1991,Heckerman1995}, or the \emph{BDeu score} when parameterized by a single constant called the \emph{equivalent sample size}. 
These choices enable efficient computation of $P(G | X)$ for any given $G$ up to a normalizing constant, as $\pi(G) = \prod_i \pi_i(A_i)$. We call $\pi_i(A_i)$ \emph{local scores}. 

We stress that our algorithms also apply to other models, including ones for continuous variables. For computational efficiency, it is desirable that the marginal likelihood term can be efficiently computed for each node given its parents. This holds, e.g., for the commonly used BGe score \citep{Geiger2002,Kuipers2014} for linear Gaussian models. 
However, unlike BD that is flexible enough to be meaningfully applied to continuous data via discretization, the model underlying BGe is more restrictive and lacks a comparable non-parametric character.

\subsection{Sampling along Markov Chains}

The Metropolis--Hasting (MH) algorithm \citep{Hastings1970} gives us a way to construct a Markov chain on the state space of DAGs such that its stationary distribution is $P(G | X)$. First, an initial state $G_0$ is generated arbitrarily. Then, iteratively, a new state $G'$ is in step $t+1$ drawn from a proposal distribution $Q(G' | G_t)$ and accepted as the next state $G_{t+1}$ with probability
$\min\!\big\{1, R(G', G_t)\big\}$, where 
\be\label{eq:MH}
 	R(G', G_t) = \frac{\pi(G') Q(G_t| G')}{\pi(G_t) Q(G' | G_t)}
\ee
is the untruncated acceptance ratio. If the proposal is rejected, $G_{t+1}$ is set to $G_t$. 

While very mild conditions on the proposal distribution $Q$ suffice for guaranteeing convergence to the posterior, a good choice of $Q$ can significantly improve the mixing of the chain and, thereby, expedite convergence. Typically $Q(G'| G_t)$ is chosen to have a restricted support: the chain can move only to some \emph{neighbors} of $G_t$, as specified by the available move types. We next review move types proposed in the literature, in increasing order of complexity. 

The \emph{basic moves} consist of adding, deleting, or reversing a single arc. A naive implementation would explicitly construct the neighborhood of $G_t$ and then draw one neighbor uniformly at random. This is computationally demanding, as typically one could add $\Omega(n^2)$ different arcs. Moreover, to check whether a candidate arc can actually be added, one has to check for acyclicity of the resulting graph: effectively, one has to check whether the head node of the arc is an ancestor of the tail node in the original graph, which can take time $\Omega(\ell)$ in a graph with $\ell$ arcs. 

\citet{Giudici2003} address the computational challenge by maintaining the ancestor relation in an $n \times n$ binary matrix. Node pairs are tentatively proposed for addition or reversal---acyclicity is checked afterwards. Testing whether an arc $ij$ can be added or reversed takes, respectively, $O(1)$ or $O(\ell_j)$ time, where $\ell_j$ is the number of parents of $j$. Only after accepting a move, the matrix is updated, taking $O(n \ell)$ time.\footnote{\citet{Giudici2003} claim $O(n)$ update time after addition and $O(n^2)$ after removal; however, the correct bounds for their algorithms appear to be $O(n^2)$ and $O(n \ell)$.} 

\citet{Grzegorczyk2008} presented a more involved \emph{edge reversal} (REV) move. Not only an arc $ij$, selected uniformly at random, is reversed, but also the parent sets of both end nodes $i$ and $j$ are resampled based on the local scores and ensuring acyclicity. The ratio \eqref{eq:MH} for moving from $G = (N, A)$ to $G' = (N, A')$ becomes
\bes
	R(G', G) = \frac{|A|}{|A'|}\, \frac{\Zt{ji}'}{\Zt{ij}}\, \frac{\Zh{ji}'}{\Zh{ij}}\,,
\ees
where each of the four \emph{Zustandssumme}-terms is the normalizing constant for drawing the respective parent set.  
More precisely, letting $U_i$ denote the set of \emph{non-descendants} of node $i$ in $G$, we have 
\bes
	\quad
	\Zt{ij} = \sum_{S \subseteq U_i} \pi_i(S)\,, \quad
	\Zh{ij} = \sum_{i \in S \subseteq U_j} \pi_j(S)\,,
\ees
and similar sum formulas for $\Zh{ji}'$ and $\Zt{ji}'$. 

\citet{Su2016} took the idea of resampling parent sets further. Their \emph{Markov blanket resampling} (MBR) move first draws a random target node $i$ and a random ordering of its children. 
Next the arcs to $i$ are removed, as well as the arcs to each child $j$ of $i$, with the exception of the arcs starting from $i$. 
Then a new parent set of $i$ disjoint from the original $A_i$ is resampled subject to the acyclicity requirement. 
Last, new parents are sampled for each child one by one, in the fixed ordering, ensuring that $i$ is one parent and the resulting graph is acyclic. 

For a more precise description, denote by
$G_{(j)}$ the graph obtained from $G$ by removing all arcs to child~$j$ of $i$ and to all nodes succeeding $j$ in the ordering. Denote by $U_{(j)}$ the set of non-descendants of $j$ in $G_{(j)}$.  
Then the ratio \eqref{eq:MH} can be written as  
\bes
	R(G', G) = \frac{Z_i'}{Z_i}\, \prod_{j \in C_i} \frac{Z_{(j)}'}{Z_{(j)}}\,,
\ees
where $C_i$ is the set of children of $i$, and 
\bes
	Z_i = \sum_{S \subseteq U_i \setminus A_i} \pi_i(S)\,, \quad
	Z_{(j)} = \sum_{i \in S \subseteq U_{(j)}} \pi_j(S)\,,
\ees
and similar sum formulas for $Z_i'$ and $Z_{(j)}'$. 

While REV and MBR moves can take longer leaps in the space of DAGs (with reasonable probability), compared to basic moves, they are also computationally more demanding. Even if the number of parents per node is bounded by some constant $d$, the $Z$-terms generally involve $\Omega(n^d)$ parent sets.  

This is the bottleneck also in yet another similar move proposed by \citet{Goudie2016}. Their block update move resamples the parent sets for a few (say, three) randomly selected nodes. It differs from REV and MBR in that Gibbs sampling is used. 

\section{FAST BASIC MOVES}
\label{se:gibby}

This section presents a novel algorithm that, like the one by \citet{Giudici2003}, performs three basic moves: removing, adding, and reversing an arc. Although the two algorithms produce the same Markov chain, they differ entirely in their architecture and in the proposal distributions they use. Both methods rely on a rejection sampling approach: we may propose moving from the current graph to a graph that is cyclic but, of course, reject such proposals after discovering cyclicity. The distinctive feature of our algorithms is that acyclic proposals always get accepted! Thus, we do not waste simulation steps for proposing low-scoring DAGs that would get rejected, like the GC algorithm does. We call our algorithm \ourbasic{}, as it bears some resemblance to Gibbs sampling.

\subsection{Enhanced Simulation}

Our starting point is a simple MH algorithm. It first draws a node pair $ij$ according to a fixed probability distribution $(q_{ij})$, independent of $G$. Then it proposes a move to $G^{ij}$, obtained from $G$ either by removing arc $ij$ (if present in $G$), reversing arc $ji$ (if present in $G$), or adding arc $ij$ (if neither $ij$ nor $ij$ is in $G$). Finally, $G^{ij}$ is accepted as the next state of the chain with probability
\bes
	\alpha(G^{ij}, G) = \min\bigg\{1, \frac{\pi(G^{ij})}{\pi(G)}\bigg\}\,,
\ees
where $\pi(G^{ij}) = 0$ if $G^{ij}$ is cyclic; otherwise the next state is $G$. It is easy to show that if $q_{ij} = q_{ji}$, the chain converges to its stationary distribution $\pi$. In our experiments, we used the uniform distribution, $1/q_{ij} = n(n-1)$.
With some other choice, one may concentrate proposals on important node pairs, e.g., pairs that are more likely adjacent in the DAG posterior.

A major weakness of the above algorithm is that it may frequently propose low-scoring graphs $G^{ij}$, which get rejected with high probability. Ideally, we would generate $G^{ij}$ with probability $a_{ij} := q_{ij} \cdot \alpha(G^{ij}, G)$. With probability $a := \sum_{ij} a_{ij}$ the chain would move to some $G^{ij}$ and with probability $1-a$ it would stay at $G$. Unfortunately, we do not know how to implement this without explicitly visiting each possible $G^{ij}$, which would be computationally demanding.  

We give an efficient variant of the above scheme. The key is to consider \emph{tentative acceptance probabilities}
\bes
	\beta(G^{ij}, G) = \min\bigg\{1, \frac{\pi^*(G^{ij})}{\pi(G)} \bigg\}\,,
\ees
where $\pi^*$ is obtained from $\pi$ by extending the function to all directed graphs, i.e., ignoring the acyclicity constraint. In particular, if $G ^{ij} = (N, A')$, with $A'$ possibly cyclic, we have 
\bes
	\frac{\pi^*(G^{ij})}{\pi(G)} 
	= \frac{\pi_i(A'_i)}{\pi_i(A_i)}\,\frac{\pi_j(A'_j)}{\pi_j(A_j)}\,.
\ees
Note that if $G^{ij}$ is obtained by deleting or adding arc $ij$ from node $i$ to node $j$, then $A'_i = A_i$ and the respective ratio cancels out, leaving just the ratio $\pi_j(A'_j)/\pi_j(A_j)$.

Now, let
\bes
	b_{ij} := q_{ij} \cdot \beta(G^{ij}, G)\,,\quad
	b := \sum_{ij} b_{ij}\,.
\ees
To simulate the next state after $G$, we let the chain stay at $G$ with probability $1-b$, and with the remaining probability $b$ we let the chain move to the graph returned by the following algorithm.

\begin{description}
\item[Algorithm~B] 
\item[B1]
Draw a node pair $ij$ with probability $b_{ij}/b$.
\item[B2]
If $G^{ij}$ is acyclic, then return $G^{ij}$; else return $G$.
\end{description}

Observe this is a MH algorithm with the proposal distribution $Q(G^{ij} | G) = b_{ij}$ for any acyclic $G^{ij} \neq G$. The untruncated acceptance ratio $R(G^{ij}, G)$ thus equals $1$, since $\beta(G, G^{ij}) \pi(G^{ij}) = \beta(G^{ij}, G) \pi(G)$. It follows that $\pi$ is the stationary distribution of the chain, as desired. (See Suppl.~\ref{se:convergence} for a more detailed proof.)

When $b$ is small, we may simulate multiple consecutive steps simply by drawing a geometric random variable with mean $1/b$ (our present implementation), or attach each state with a weight equalling $1/b$, similar to a related birth-death process formulation for sampling undirected networks \citep{Mohammadi2015}. In contrast, for Gibbs sampling it is desirable to lower self transition probabilities \citep{Neal2024}.

\subsection{Fast Generation of Proposals}

For drawing node pairs $ij$ with probabilities $b_{ij}/b$, we partition the pairs into groups $N_j = \{ ij : i \neq j \}$, each with its partial sum $b_j = \sum_{i} b_{ij}$. We draw a pair $ij$ in two phases: first $j$ with probability $b_j/b$, and then $i$ from $N_j$ with probability $b_{ij}/b_j$. Using a sum-tree data structure (Suppl.~\ref{se:sumtree}) for each of the $n$ distributions, generating $ij$ takes only $O(\log n)$ time.

The motivation for the partitioning stems from the need to efficiently update the distribution after adding, removing, or reversing an arc. We make use of the fact that an operation on node pair $ij$ only affects the probabilities $b_{uv}$ in which either $u$ or $v$, or both, belong to $\{i, j\}$. After a removal the update is as below; after an addition it is identical; after a reversal we first remove the arc and then add the reversed arc.

\begin{table*}[t]
\caption{Speed of \ourbasic{} vs.\ \emph{GC} on datasets of size 1,{}000 sampled from benchmark Bayesian networks}
\label{table:versusgc1000}
\medskip
\centering
{\small
\begin{tabular}{rcccccccc}
\toprule
 &       &      &     & \multicolumn{3}{c}{Steps per microsecond} & \multicolumn{2}{c}{Acceptance ratio (\%)}  \\
\cmidrule(l{4pt}r{4pt}){5-7} \cmidrule(l{4pt}r{4pt}){8-9}
Network & Nodes & Arcs & Max-indeg & \GC{} & \ourbasic{}$^{\dagger}$ & \ourbasic{}$^{\ddagger}$ & $a$ & $b$  \\
\midrule 
\alarm	&  37 &  46 & 4 & 3.0 & {\bf67} & 38 & 0.50 & 4.6  \\
\hailfinder	&  56 &  66 & 4 & 3.1 & {\bf91} & 83 & 0.31 & 1.2  \\
\pathfinder	& 109 & 195 & 5 & 1.9 & 18 & {\bf19} & 0.34  & 1.3 \\
\andes 		& 233 & 338 & 6 & 2.0 & 46 & {\bf58} & 0.058 & 0.19  \\
\pigs 		& 441 & 592 & 2 & 3.1 & 553 & {\bf582} & 0.017& 0.14   \\
\bottomrule
\multicolumn{5}{l}{\scriptsize $^{\dagger}$Constant-time acyclicity checks by maintaining the ancestor relation} &
\multicolumn{3}{l}{\scriptsize $^{\ddagger}$Acyclicity checks by path-finding algorithms}

\end{tabular}
}
\end{table*}

\begin{description}
\item[Algorithm~U](Update after removing arc $ij$) 
\item[U1]
Recompute $b_{ij}$ and $b_{ji}$.
\item[U2]
Recompute $b_{uj}$ for all $u \neq i, j$.
\item[U3]
Recompute $b_{jv}$ for all parents $v \neq i$ of $j$. 
\item[U4]
Recompute the affected marginals $b_j$, $b_v$, and $b$. 
\end{description}

Steps U1 and U2 are straightforward to compute in time $O(n)$, and step U3 in time $O(\ell_j)$, where $\ell_j$ is the number of parents of $j$. In step U4, we construct binary sum-trees for the distributions $(b_{uj}/b_j)_u$ and $(b_u/b)_u$ in time $O(n)$. Furthermore, we update the binary sum-trees for the distributions $(b_{uv}/b_v)_u$ for each of the $\ell_j$ updates of $b_{jv}$, each in time $O(\log n)$.  

\subsection{Fast Acyclicity Checks}

Recall that the \GC{} algorithm of \citet{Giudici2003} maintains the ancestor relation to enable constant-time acyclicity checks. This is motivated by the low acceptance rate of the algorithm: one expects a large number of acyclicity checks per accepted move. 

But \gibby{} has a significantly higher acceptance rate, since only acyclic proposals are rejected. This shifts the optimal tradeoff between the work per proposed graph and the work per accepted move. While sophisticated methods are known for balancing the query and update time \citep{Hanauer2020}, we have found that simple path-finding algorithms (such as depth-first or breadth-first search) are sufficiently fast to not dominate the overall computation time. The worst-case complexity is linear in the number of arcs, $O(\ell)$. This should be contrasted with \GC{}'s $O(n \ell)$ complexity per update (after arc removal). 

As noted by the original authors, however, the hidden constant factor is small for \GC{}, thanks to 64-bit processors being able to handle 64 nodes by a single operation. Modern vector processing capabilities (SIMD instructions) may boost this further. Acknowledging this, we have also implemented a version of \gibby{} that takes the approach of \GC{} for acyclicity checking.

\subsection{Performance in Practice}

We compared our somewhat optimized \cpp{} implementations of \GC{} and \ourbasic{} on datasets drawn from five benchmark BNs. We used the sparse prior with the maximum-indegree of the data-generating network. Local scores were precomputed only for parent sets of sizes $0$ and $1$; for larger parent sets, local scores were computed on demand and cached. For each dataset, we ran both \GC{} and \ourbasic{} for $2 \times 10^{10}$ steps, gauging the running time only for the latter half to lessen the influence of the demanding score computations in the burn-in phase of the simulation. As \GC{} and \ourbasic{} simulate exactly the same Markov chain, they differ only in speed, the mixing properties being the same. Table~\ref{table:versusgc1000} shows results on datasets of size 1,{}000 (see Suppl.~\ref{se:gibbymore} for sizes 500 and 2,{}000). 

Our implementation of \GC{} takes $2$--$3$ million steps per second, across the tested datasets, making \emph{billion}-step simulations computationally feasible. For comparison, in the literature \citep{Giudici2003,Grzegorczyk2023,Grzegorczyk2024,Grzegorczyk2008,Su2016} the reported simulations are in a few million steps. The larger the network, the lower the acceptance rate (value $a$) and the less frequent updating of the data structures. While the running time per step is about the same, the number of accepted moves is significantly smaller for larger networks.

\gibby{}, in contrast, takes advantage of low acceptance rate, reflected in the tentative acceptance rate (value~$b$). On the largest of the tested networks, \ourbasic{} is about 200 times faster than $\GC{}$, rendering \emph{trillions} of steps computationally feasible; also on the smaller networks, we observe at least an order-of-magnitude speedup. Of the two alternative ways to manage acyclicity, maintaining the ancestor relation has a clear advantage on the smaller networks, while calling a path-finding algorithm appears to be faster on the larger networks. This can be explained by the time complexity of updating the ancestor relation: even if is needed less frequently on larger networks, the update cost grows relatively fast with the number of nodes.  

\section{SCORE PRUNING}
\label{se:pruning}

The REV and MBR moves are computationally more expensive than the basic moves because they require summing local scores over possible parent sets for one or multiple nodes. Given a node $i \in N$ and two sets $T \subseteq U \subseteq N\setminus\{i\}$, we need the \emph{interval sum} $\pi_i[T, U] := \sum_{T \subseteq S \subseteq U} \pi_i(S)$. For REV and MBR, the lower set $T$ is either empty or a singleton set. 
To answer such interval queries $(i, T, U)$, a common approach is to loop over \emph{all} possible parent sets $S$ of $i$ in decreasing order by the scores, accumulate the sum by $\pi_i(S)$ if $T \subseteq S \subseteq U$, and stop as soon as the contributions of the remaining scores are guaranteed to be only  negligible \citep{Friedman2003, Niinimaki2013, Niinimaki2016, Viinikka2020}.

We next present a way to shorten the list of local scores by safely discarding parent sets that are dominated by other sets in the list. Our pruning method is inspired by the folklore pruning rule \citep{Teyssier2005} that applies when one seeks a DAG that \emph{maximizes} the score; then one can discard set $S$ if it has a subset $R$ with an equal or better score, that is, $\pi_i(S) \leq \pi_i(R)$. Our pruning rule is similar in spirit but technically more involved, as we have to consider approximate evaluations of interval queries. 

\subsection{Theory}

Let us focus on a fixed child node $i$. For brevity, write $V$ for the set $N\setminus\{i\}$ and $f$ for the local score function $\pi_i$. We formulate the goal of approximation as a guarantee in the relative error.

\smallskip
\begin{definition}[$\epsilon$-close]\emph{
Let $V$ be a finite set and $\epsilon \geq 0$. We say that $\tilde{f} : 2^V \rightarrow \mathbb{R}_{\ge 0}$ is \emph{$\epsilon$-close} to $f : 2^V \rightarrow \mathbb{R}_{\ge 0}$ if 
$(1 - \epsilon) f[T, U] \leq \tilde{f}[T, U] \leq f[T, U]$
for all $T, U \subseteq V$ with $|T| \leq 1$.
}
\end{definition}

Here, only allowing underestimation of the exact sum is motivated by our pruning method that forms $\tilde{f}$ by simply discarding some sets $S$, i.e., setting $\tilde{f}(S)$ to $0$.

\smallskip
\begin{definition}[$\epsilon$-pruning]\emph{
Let $V$ be a finite set of $K$ elements and $\epsilon \geq 0$. Let $f : 2^V \rightarrow \mathbb{R}_{\ge 0}$. The \emph{$\epsilon$-pruning} of $f$ is the function $\tilde{f} : 2^V \rightarrow \mathbb{R}_{\ge 0}$ obtained as 
\bes
	\tilde{f}(S) := 
		\begin{cases}
			0\,, & \text{if $f(S) < \epsilon \cdot \psi(j, S)$ for all $j \in S$},\\ 
			f(S)\,, & \text{otherwise},\\
		\end{cases}
\ees 
where
\bes
	\psi(j, S) 
	= \sum_{j \in R \subseteq S} f(R) \big(1 + 1/K\big)^{|R| - K} K^{|R| - |S|}\,.
\ees
}
\end{definition}

In words, discard $S$ if its score is much less than sums of appropriately weighted scores of its subsets. 
\smallskip
\begin{theorem}[Suppl.~\ref{se:proof}]\label{thm:pruning}
The $\epsilon$-pruning of $f$ is $\epsilon$-close to $f$.
\end{theorem}

We can show (Suppl.\ Theorem~\ref{thm:lostmass}) that with $(\epsilon/n)$-pruning of $\pi_i$ for all nodes $i$, the Kullback--Leibler divergence and the total variation distance of the resulting approximate posterior from the exact one are at most $\epsilon/(1-\epsilon)$ and $\epsilon$, respectively. 

\subsection{Practice}

The $\epsilon$-pruning $\tilde{\pi}_i$ of a score function $\pi_i$ is straightforward to compute: First, $\pi_i(S)$ is computed for all possible parent sets $S$. Then, by a second pass, each set $S$ is either pruned or kept depending on whether $\pi_i(S)$ is less than $\epsilon\cdot \psi_i(j, S)$ for all $j \in S$; here $\psi_i$ is defined as $\psi$ but, of course, separately for each node~$i$. For fixed $i$ and $j$, the evaluation of $\psi_i(j, S)$ takes time $O\big(2^{|S|}\big)$. We call this \emph{complete} pruning, since it computes the scores also for all sets that get pruned.

Ideally, one would only score the sets that will be kept. When pruning significantly reduces the number of kept sets, such lazy computation could be much faster than complete pruning.
We have partially implemented this idea in what we call \emph{bottom-up} pruning, as follows. We visit possible parent sets $S$ in increasing order by size. As before, we prune $S$ if $\pi_i(S)$ is less than $\epsilon\cdot \psi_i(j, S)$ for all $j \in S$. However, now we also prune $S$ if \emph{all} its subsets $S\setminus\{j\}$ for $j \in S$ were pruned. Thus, for example, if all sets of size $5$ are pruned, there is no need to score sets of size $6$ or larger. 
While we cannot prove that this heuristic never prunes any extra set, we have empirically observed that complete and bottom-up pruning typically yield exactly the same result. 

With hundreds of nodes and relatively large maximum-indegree bound, say $5$ or $6$, further heuristics may be needed for computational feasibility, both concerning the time requirement of score computations and the memory requirement for storing the scores. To this end, we adopt the idea of preselecting some number $K$ of candidate parents per node \citep{Friedman2003,Kuipers2022,Viinikka2020}; as the candidate parents of node $i$, we simply selected the $K$ nodes $j$ with the largest scores $\pi_i(\{j\})$ \citep{Friedman2003}, but also other heuristics could be considered \citep{Viinikka2020}, e.g., through constraint-based methods \citep{Kuipers2022}. In addition, we also consider parent sets not contained in the set of candidates, however, subject to a lower maximum-indegree bound, which we set to roughly match the amount of available memory.

\begin{table*}[t!]
\caption{Efficiency of pruning on datasets of size 1,{}000 sampled from benchmark Bayesian networks} %
\label{table:pruning}
\medskip
\centering
{\small
\begin{tabular}{rcccw{c}{55pt}w{c}{30pt}w{c}{30pt}w{c}{55pt}w{c}{30pt}w{c}{30pt}}
\toprule
 & & & & \multicolumn{3}{c}{$\epsilon = 2^{-15}$} & \multicolumn{3}{c}{$\epsilon = 2^{-10}$} \\
\cmidrule(l{4pt}r{4pt}){5-7} \cmidrule(l{4pt}r{4pt}){8-10}
Network & Nodes & \!\!\!Max-indeg & $K$$^{\dagger}$ & Kept (\%) & \TimeB$^{*}$ & \TimeC & Kept (\%) & \TimeB & \TimeC \\
\midrule 
\alarm		&  37 &  4 & 36 & 5.9 & 8.9 & 15 & 3.7 & 7.8 & 15 \\
\hailfinder	&  56 &  4 & 55 & 0.090 & 6.5 & 127 & 0.087 & 6.4 & 127 \\
\pathfinder	& 109 &  5 & 64 & 22 & 9230 & 17800 & 15 & 8450 & 17100 \\
\andes		& 223 &  6 & 64 & 0.16 & 10300 & -- & 0.041 & 3470 & -- \\
\pigs		& 441 &  2 & 64 & 18 & 11 & 11 & 16 & 11 & 11 \\
\pigs		& 441 &  \hspace{4pt}4$^{\ddagger}$ & 64 & 0.11 & 125 & 5490 & 0.10 & 116 & 5640 \\

\bottomrule
\multicolumn{10}{l}{\scriptsize 
$^{\dagger}$No.\ candidate parents per node\;
$^{*}$No.\ seconds by bottom-up (\TimeB) and complete pruning (\TimeC)\;
$^{\ddagger}$Larger than in the generating network
}

\end{tabular}
}
\end{table*}

Table~\ref{table:pruning} shows empirical results on datasets of size 1,{}000 generated from five benchmark BNs (see Suppl.~\ref{se:pruningmore} for sizes 500 and 2,{}000). 
Our main observation is that the efficiency of pruning diverse in scale across the datasets, and that the variability is not easily explained by the key parameters of the generating networks, such as the number of nodes or arcs, or maximum-indegree. That said, a larger maximum-indegree gives better chances for significant pruning to take place; this holds especially when we use a maximum-indegree larger than the actual maximum-indegree of the generating networks (second instance of \pigs{}).
Aligned with expectations, bottom-up pruning is much faster than complete pruning when the fraction of parent sets kept after pruning is small. 
The results appear to be insensitive to the relative error parameter $\epsilon$. The only exception is \andes{}, for which increasing $\epsilon$ from $2^{-15}$ to $2^{-10}$ reduced the kept parents to about one fourth and the running time of bottom-up pruning to about one third.

\section{EMPIRICAL COMPARISON TO OTHER SAMPLERS}
\label{se:comparison}

We compared our algorithm \gibby{} to two recent Bayesian DAG samplers, \bidag{} \citep{Suter2023} and \daggflow{} \citep{Deleu2022}, both of which can handle categorical data and support modular (e.g., uniform) structure priors. We excluded from the comparison other samplers that cannot express these models (see Suppl.~\ref{se:samplers} for a short review of related methods). 

\bidag{} is a sophisticated implementation of the partition MCMC method \citep{Kuipers2017}. It simulates a Markov chain on ordered node partitions, each which corresponds to exponentially many DAGs.
Various ideas are employed to find a small set of candidate parents for each node, which enable fast access to sums of local scores (similar to the interval sums of Section~\ref{se:pruning}) using precomputed look-up tables. 

\daggflow{} samples DAGs using generative flow networks in a fashion similar to sequential importance sampling. It first learns a ``transition probability distribution'' for moving from a given DAG to another DAG with one additional arc. Then it can generate DAGs independently from an approximate posterior. 

For all the three samplers, we used the sparse structure prior and the Dirichlet priors for the parameters with the equivalent sample size $1.0$. The maximum-indegree parameter was set separately for each dataset, as described in Suppl.~\ref{se:accuracymore}. 
For more details on our the choices for the user parameters and on the computing environment, see Suppl.~\ref{se:samplers} and \ref{se:environment}.

\subsection{Small Networks}

We ran the three samplers on datasets of size 1,{}000 sampled from selected small benchmark BNs \citep{ScutariBNrepository}, as well as on the \zoo{} benchmark dataset of size 101 \citep{Zoo}. We measure the accuracy of a sample of DAGs by the \emph{MAD}, i.e., the \emph{maximum absolute deviation} $|\hat{p}_{ij} - p_{ij}|$ over all node pairs $ij$. Here $\hat{p}_{ij}$ is the relative frequency of arc $ij$ in the sampled DAGs and $p_{ij}$ is the respective exact probability we compute by an exponential-time algorithm \citep{Harviainen2024,Pensar2020,Tian2009} (see Suppl.~\ref{se:exact}).

We observe  (Fig.~\ref{fig:small}) that the arc posterior probability estimates by \gibby{} quickly become close to the best possible, the MAD decaying roughly at the rate $O\big(1/\!\sqrt{t}\big)$ in the sample size $t$. The estimates by \bidag{} improve as $t$ increases for \asia{}, but for the other networks, the MAD stays at high values, suggesting that the Markov chain does not mix well; further evidence for this is given in Suppl.~\ref{se:accuracymore}. \daggflow{} performs reasonably in some runs on \asia{}, but more generally it does not produce reliable output for the tested datasets; as the approach of \daggflow{} is very different from that of \gibby{} and \bidag{}, we defer the detailed results to Suppl.~\ref{se:accuracymore}. 

\begin{figure*}[t!]
\begin{center}
{\small
\begin{tabular}{cccc}
\hspace{5pt} \asia{} ($n = 8$)  & 
\hspace{-10pt} \sachs{} ($n = 11$) & 
\hspace{-5pt} \child{} ($n = 20$) & 
\hspace{-5pt} \zoo{} ($n = 17$) \\
\addlinespace[2pt]
\hspace{-5pt}
\includegraphics[height=0.21\textwidth]{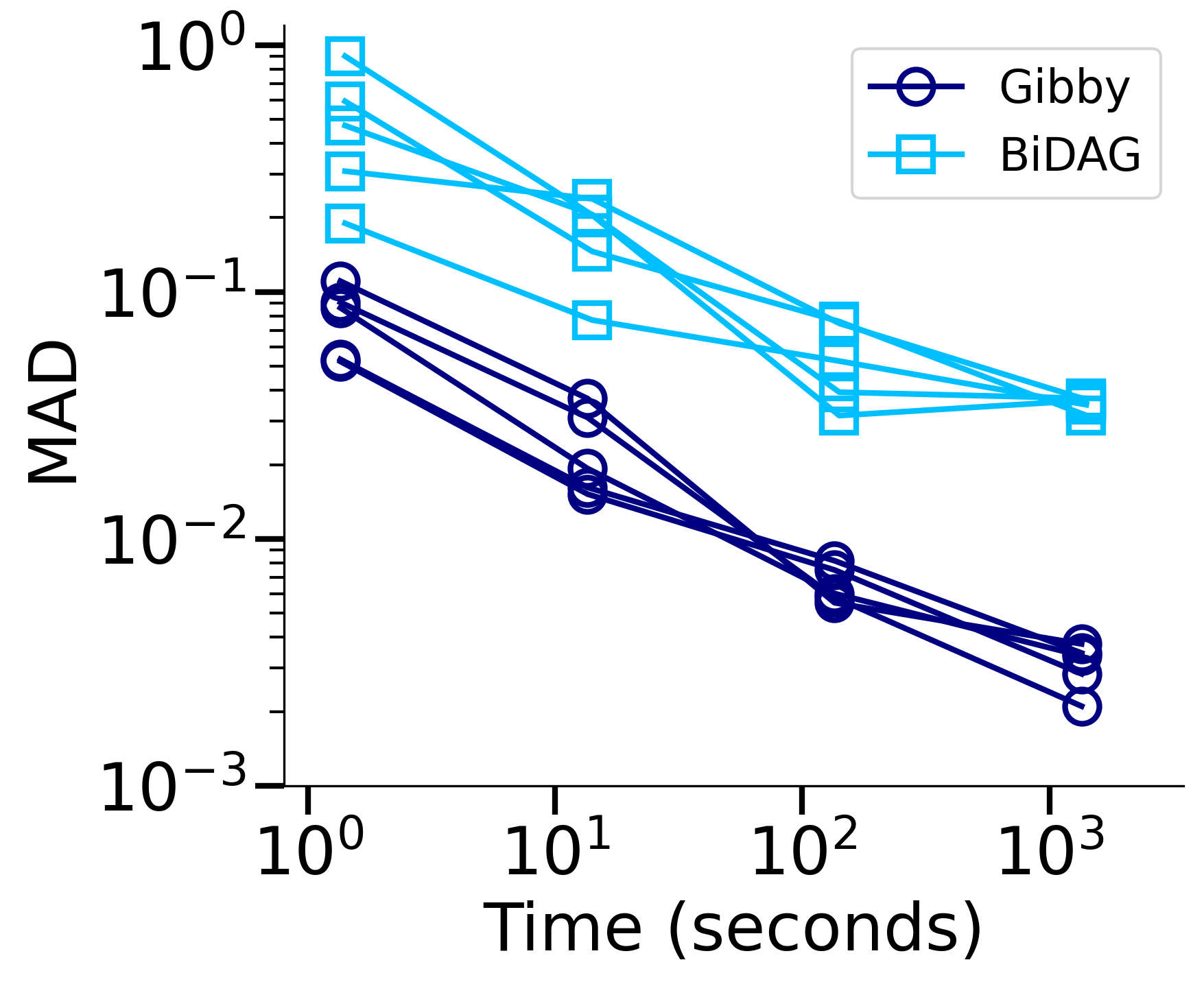} & 
\hspace{-10pt}
\includegraphics[height=0.21\textwidth]{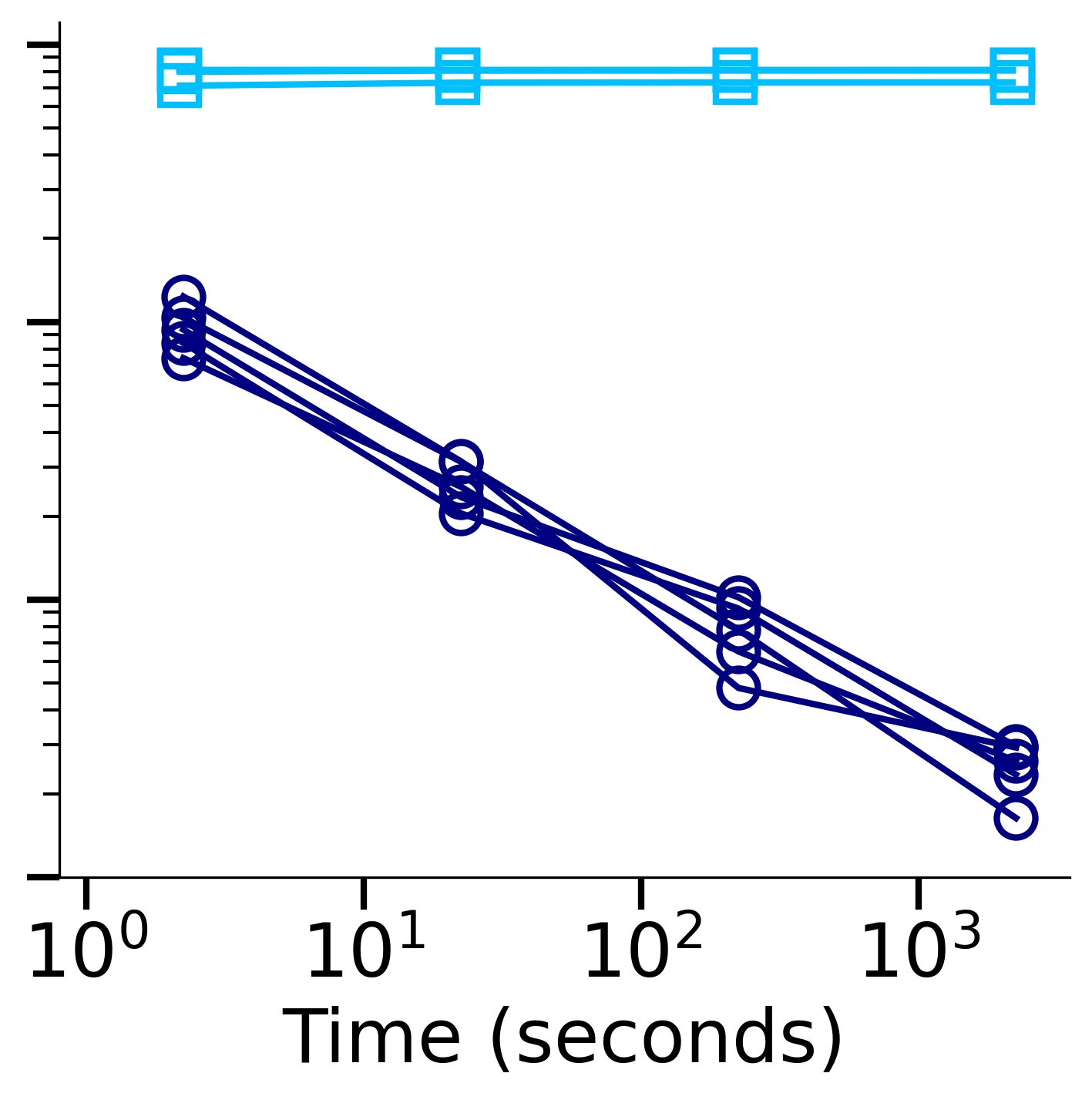} & 
\hspace{-10pt}
\includegraphics[height=0.21\textwidth]{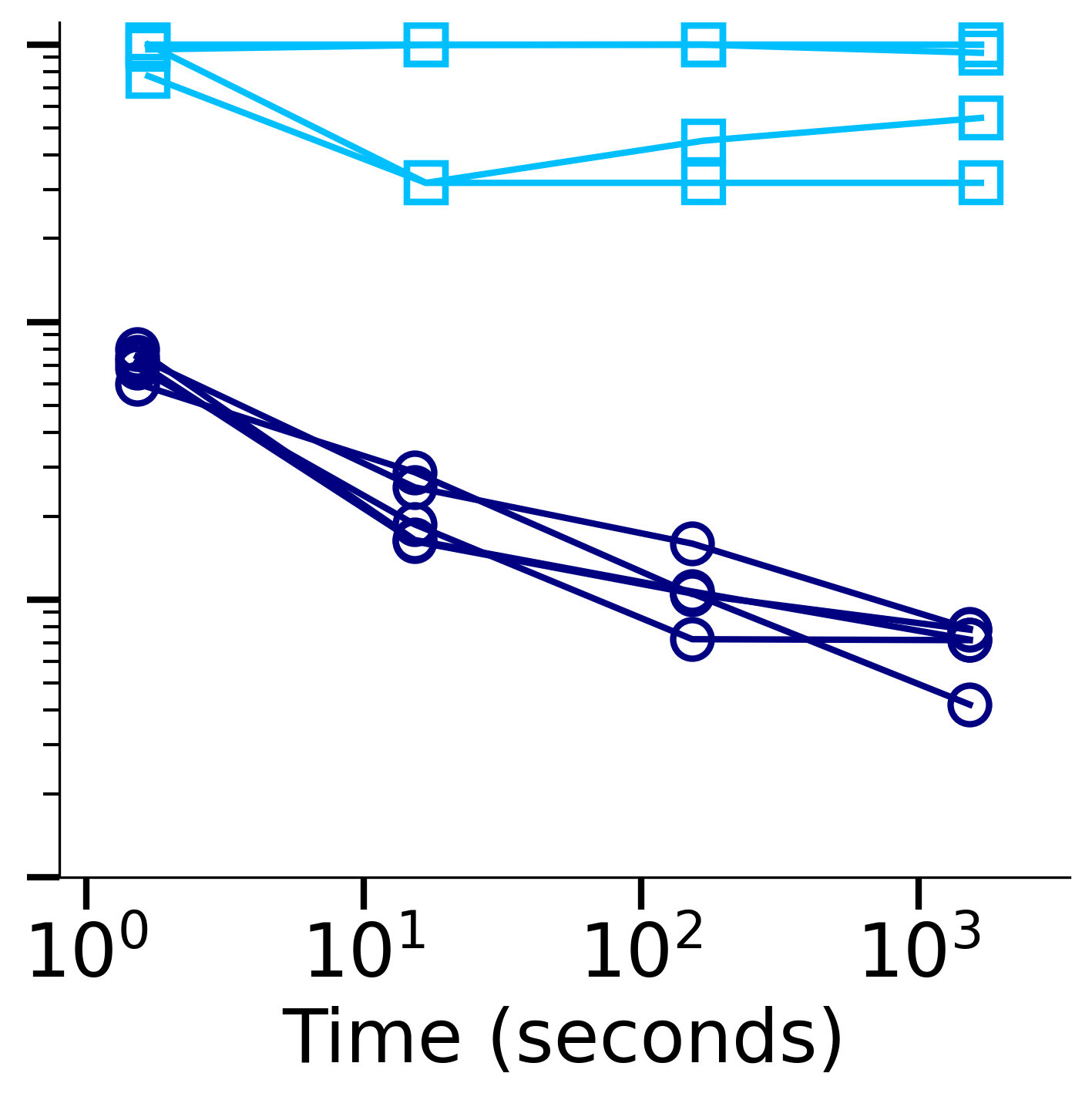} & 
\hspace{-10pt}
\includegraphics[height=0.21\textwidth]{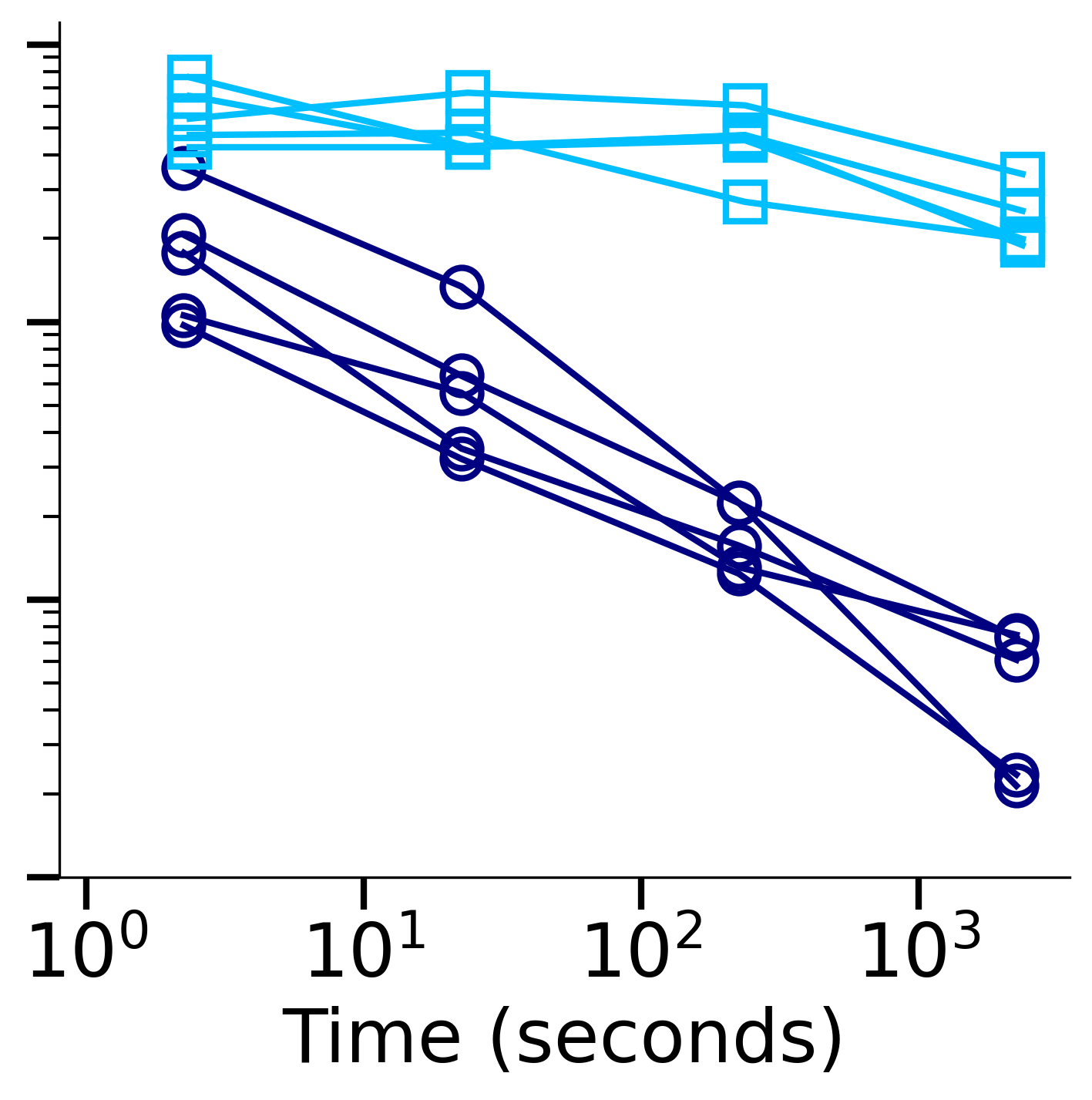}  
\end{tabular}
}
\caption{\label{fig:small}
Performance on small networks. For each sampler, shown is the MAD of five independent runs as a function of running time. For each run, 100,{}000 DAGs were collected with even spacing. 
}
\end{center}
\end{figure*}

\begin{figure*}[t!]
\begin{center}
{\small
\begin{tabular}{cccc}
\multicolumn{2}{c}{\hailfinder{} ($n = 56$)} & \multicolumn{2}{c}{\andes{} ($n = 233$)} \\
\cmidrule(lr{8pt}){1-2} \cmidrule(ll{8pt}){3-4}
\hspace{6pt} \gibby{} & \hspace{-18pt} \bidag{} & \hspace{12pt} \gibby{} & \hspace{-18pt} \bidag{} \\
\hspace{-5pt}
\includegraphics[height=0.26\textwidth]{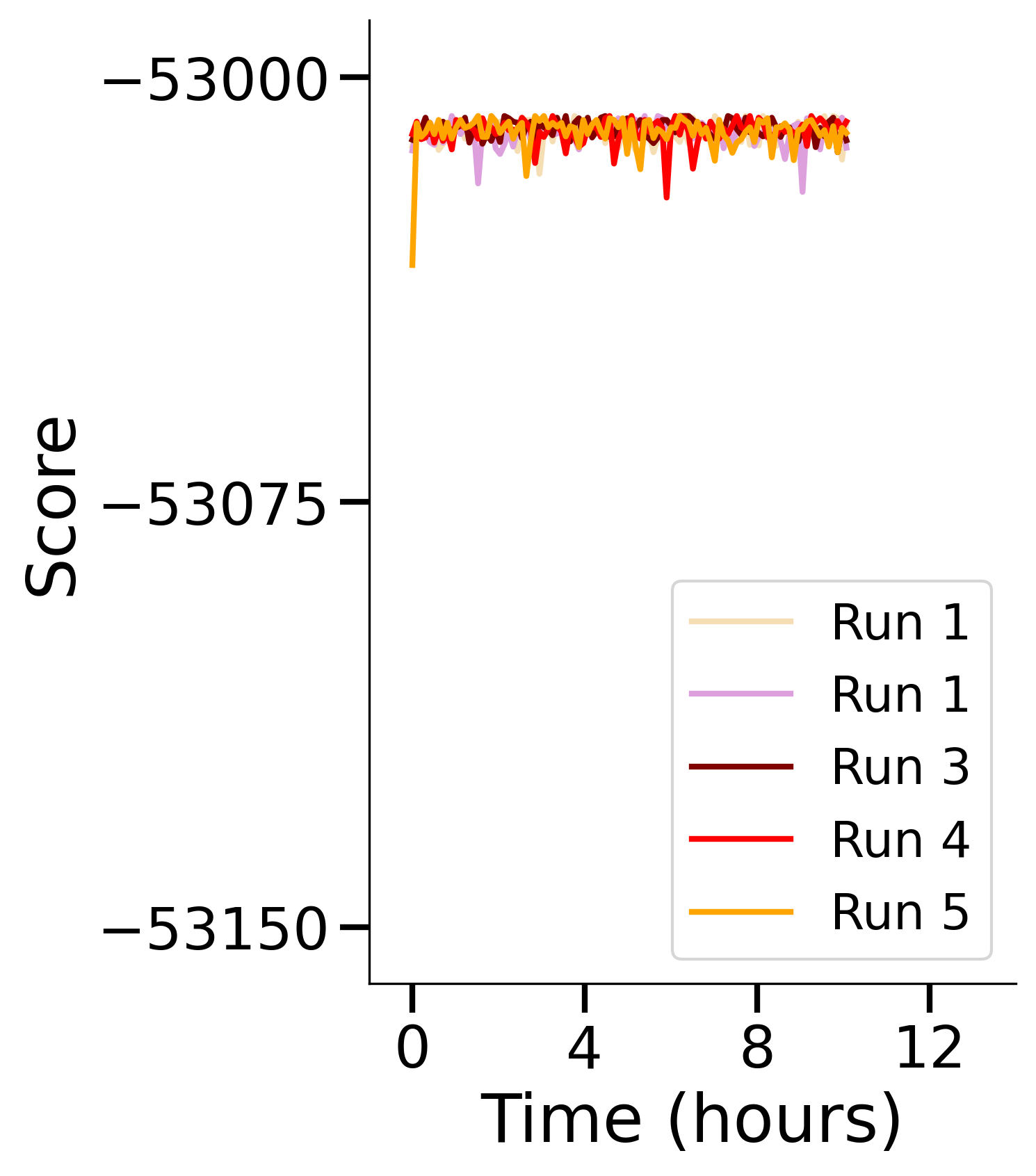} & 
\hspace{-5pt}
\includegraphics[height=0.26\textwidth]{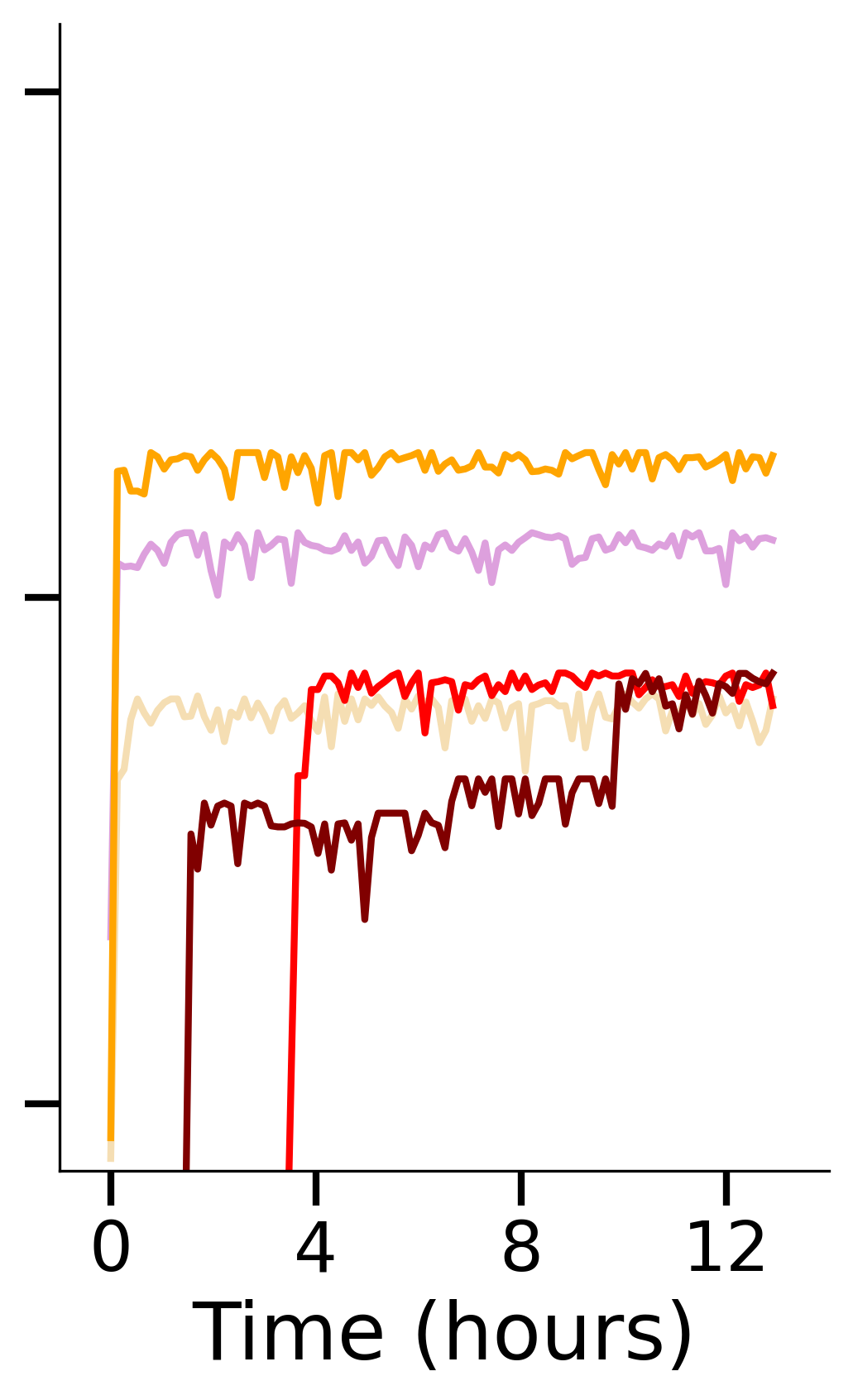} & 
\hspace{5pt}
\includegraphics[height=0.26\textwidth]{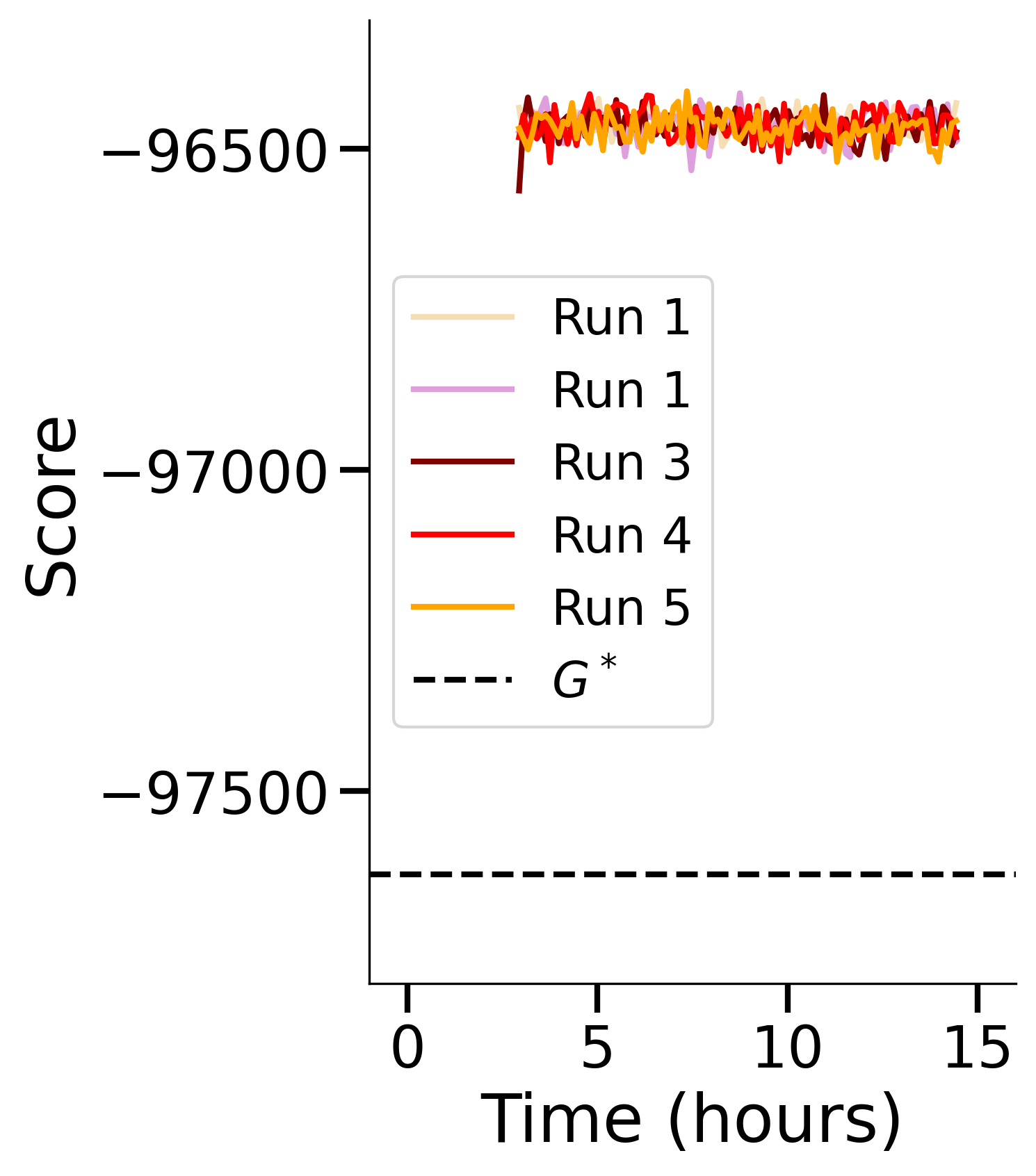} & 
\hspace{-5pt}
\includegraphics[height=0.26\textwidth]{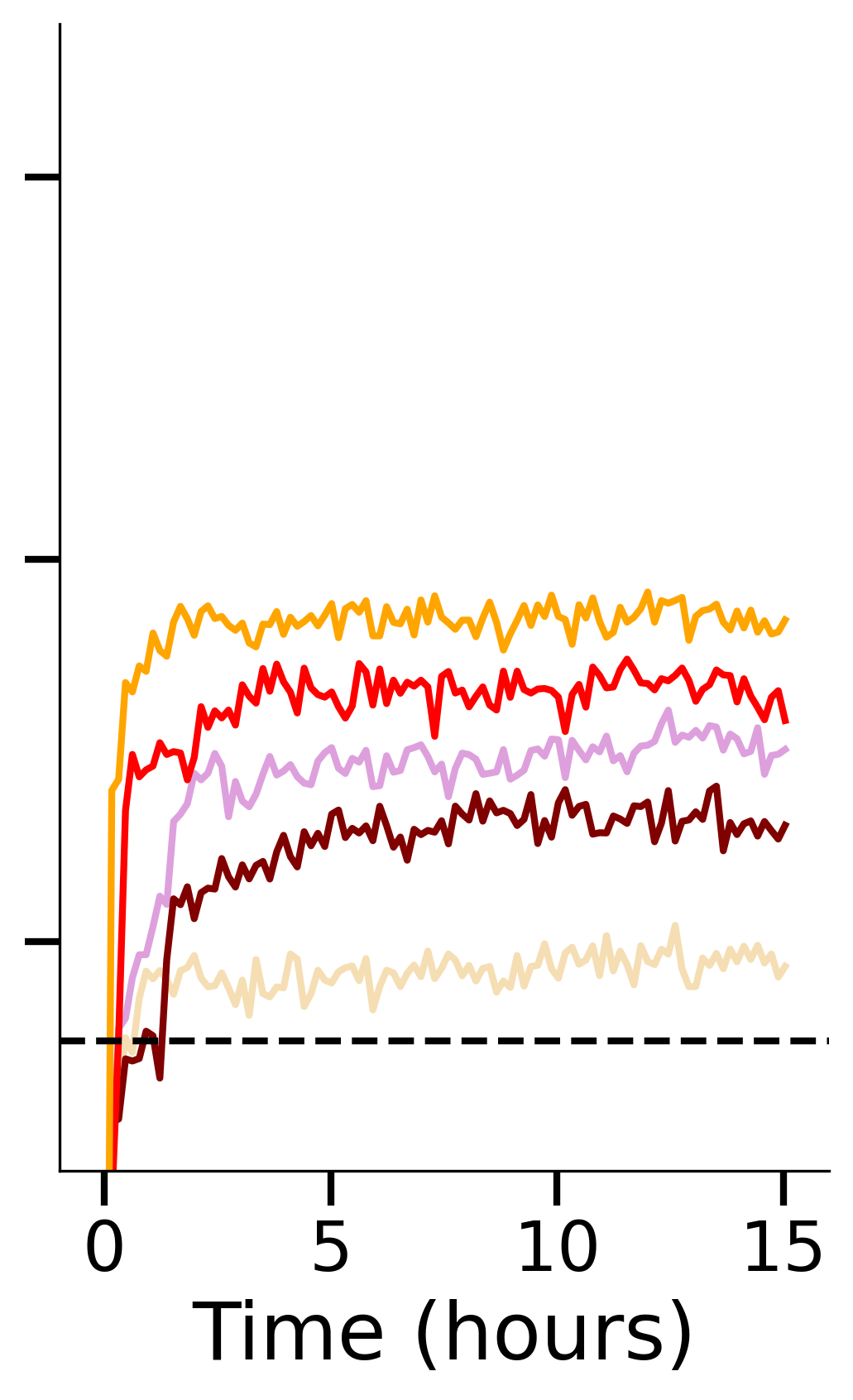} \\ 
\end{tabular}
}
\caption{\label{fig:large}
Performance on large networks. For \gibby{} the number of candidate parents $K$ were set as in Table~\ref{table:pruning}. For each sampler, shown is the log posterior probability of the sampled DAGs in five independent runs. 
For \hailfinder{}, the score of the data-generating DAG $G^*$ is around $-56010$ and not visible. For \andes{}, \gibby{} spends the first hours in score computations, before sampling. 
}
\end{center}
\end{figure*}

\subsection{Large Networks}

On large networks, we only compare \gibby{} to \bidag{}, given the weak performance of \daggflow{}. 
We generated datasets of size 1,{}000 from two benchmark BNs and visually inspect the traceplots of the samplers (Fig.~\ref{fig:large}).
We compare the sampled DAGs to the data-generating DAG $G^*$ in terms of the posterior probability, but not with respect to any structural metrics, recalling that the task is to approximate the posterior distribution (usually not well represented by $G^*$). 

We observe that both samplers quickly reach regions of DAGs whose score exceeds that of the data-generating DAG. 
\gibby{} appears to mix well, the five independent runs showing very similar behavior. 
\bidag{} is less robust, different runs often getting stuck at different local modes. 
The arc posterior probability estimates (Suppl.~\ref{se:accuracymore}, Fig.~\ref{fig:largecorr}) are highly correlated between the \gibby{} runs, whereas different \bidag{} runs yield largely varying estimates.

\section{CONCLUDING REMARKS}

We have advanced computational methods for approximate sampling of DAGs from their posterior distribution. Our results suggest that currently the best methods---those scaling beyond exact algorithms---are based on Markov chains with moves that resample entire parent sets of one or multiple nodes. However, the computational cost of such moves increases rapidly with the user-defined maximum indegree, which has limited their practical use to small DAGs or modest indegree bounds. To mitigate this, we introduced a pruning rule that can significantly reduce precomputing time, lower memory requirements, and speed up per-move computations. While pruning is highly effective in many cases, it can be less efficient on certain datasets, necessitating additional techniques or constraints to ensure computational feasibility. In such scenarios, our enhancement 
for simulating single-arc moves should prove particularly advantageous.

There are several avenues for further research. Our empirical study was limited to discrete Bayesian networks; comparison to existing samplers \citep{Annadani2023,Cundy2021} that can handle continuous and mixed data deserves a dedicated study. Naturally, a more systematic empirical study with varying dataset sizes and multiple performance metrics (for the accuracy of the posterior approximation) would be needed for getting a comprehensive understanding of the capacity and relative merits of different methods.   
Another direction is to adopt the new techniques in the partition MCMC framework \citep{Kuipers2017,Kuipers2022} to alleviate the current weaknesses of \bidag{}, namely frequent self-transitions in the space of node partitions and the lack of moves that are based on resampling parent sets. 
Finally, we believe the ideas we employed to expedite basic moves (upper bounding the acceptance ratio and simulating a simple geometric random variable) and to prune local scores (weighted domination to preserve a family approximate subset-sums) may also have applications in other contexts. 

\section*{Acknowledgements}
The research was partially supported by the Research Council of
Finland, grant~351156.

\bibliographystyle{abbrvnat}
\bibliography{paper}
\clearpage
\appendix
\onecolumn
\thispagestyle{empty}
\begin{center}
  {\LARGE \bfseries Scaling Up Bayesian DAG Sampling:\\[0.5em] Supplementary Materials}
\end{center}
\vspace{1.25em}
\renewcommand{\thesection}{\Alph{section}}
\setcounter{section}{0}

\section{ALGORITHMS AND DATA STRUCTURES}

The section proves the convergence of the chain with fast basic moves, describes the binary sum-tree data structure, proves Theorem~\ref{thm:pruning}, and gives bounds for the posterior mass lost in score pruning.

\subsection{Convergence of Fast Basic Moves}
\label{se:convergence}

We prove that the Markov chain constructed in Section~\ref{se:gibby} converges to the posterior $P(G | X) = \pi(G)/Z$, where $Z$ is a normalizing constant. To this end, it is sufficient to show that the chain is irreducible and aperiodic, and that the posterior is a stationary distribution of the chain \citep[Theorem~4.9, for example]{Levin2017}. We make the following mild assumptions:
\be 
\label{eq:assumption1}
&&\textrm{$\pi(G) > 0$ for all DAGs $G$ that contain exactly one arc}\,,\\
\label{eq:assumption2}	
&&\textrm{if $\pi(G) > 0$, then $G$ has an arc removing of which results in a DAG $G'$ with $\pi(G') > 0$}\,,\\
\label{eq:assumption3}
&&\textrm{$q_{ij} > 0$ for all node pairs $ij$}\,,\\ 
\label{eq:assumption4}
&&\textrm{$q_{ij} = q_{ji}$ for all node pairs $ij$}\,. 
\ee

Denote by $p(G^{ij} | G)$ the transition probability of the chain of moving in one step from DAG $G$ to DAG $G^{ij}$ obtained from $G$ by adding, removing, or reversing arc $ij$. Recall that, by Algorithm~B, we have 
\be \label{eq:trans}
	p(G^{ij} | G) = b \cdot \frac{b_{ij}}{b} = q_{ij} \min\bigg\{1, \frac{\pi^*(G^{ij})}{\pi(G)}\bigg\}
	= q_{ij} \min\bigg\{1, \frac{\pi(G^{ij})}{\pi(G)}\bigg\}\,.
\ee

Now, to see that the chain is irreducible, consider two arbitrary DAGs $G$ and $G''$. From $G$, we can reach the empty DAG by a sequence of arc removals, and from the empty DAG, we can reach $G''$ by a sequence of arc additions. Thus $G$ and $G''$ are connected by single-arc moves, each of which is taken with a positive transition probability \eqref{eq:trans} by assumption \eqref{eq:assumption2}.

To see that the chain is aperiodic, it suffices to show that the chain can return from the empty DAG to the empty DAG in both $2$ and $3$ steps. For example, we could add arc $ij$ and remove it, or we could add $ij$, reverse $ij$, and remove $ji$. By assumptions \eqref{eq:assumption1}--\eqref{eq:assumption3} and formula \eqref{eq:trans}, these steps have positive transition probabilities. 

Finally, to see that $\pi/Z$ is a stationary distribution of the chain, it suffices to show that the following detailed balance condition holds \citep[Prop.~1.20, for example]{Levin2017}: 
\bes
	\pi(G)p(G'|G) = \pi(G')p(G|G')
	\quad \textrm{for all DAGs $G$ and $G'$}\,.
\ees
If $G=G'$, the condition is trivial. Assume $G' = G^{ij}$ with some $ij$, for otherwise the transition probabilities vanish on both sides and the condition trivially holds. If $\pi(G) = 0$, then, by \eqref{eq:trans}, we have $p(G|G^{ij}) = 0$, and thus the condition holds. By symmetry, we may now assume that $\pi(G)$ and $\pi(G^{ij})$ are nonzero. Using \eqref{eq:trans}, we get 
\bes
	\pi(G)p(G^{ij}|G) = q_{ij} \min\big\{ \pi(G), \pi(G^{ij}) \big\}\,.
\ees
Now, if $G^{ij}$ is obtained by adding or removing $ij$, then we also have 
\bes
	\pi(G^{ij})p(G | G^{ij}) = q_{ij} \min\big\{ \pi(G^{ij}), \pi(G) \big\}\,,
\ees
and thus the condition holds. If $G^{ij}$ is obtained by reversing arc $ji$ to $ij$, then instead
\bes
	\pi(G^{ij})p(G | G^{ij}) = q_{ji} \min\big\{ \pi(G^{ij}), \pi(G) \big\}\,,
\ees
and the condition holds by assumption \eqref{eq:assumption4}. 

\subsection{Binary Sum-Tree}
\label{se:sumtree}

A binary sum-tree implements three operations:
\begin{description}
\item \texttt{Initialize}$(a)$ initializes the data structure with a vector $a = (a_1, \ldots, a_n)$ of nonnegative reals.
\item \texttt{Replace}$(i, x)$ replaces the $i$th value $a_i$ by $x$.
\item \texttt{Draw}$()$ returns a draw from the categorical distribution $(p_1, \ldots, p_n)$, where $p_i = a_i/\sum_j a_j$. 
\end{description}

The data structure is a balanced rooted binary tree with $n$ leaves. If $n$ is originally not a power of two, we pad the vector $a$ with a sufficient number of zeros.

Each node in the tree stores a value. The $i$th leaf stores $a_i$. The root and the internal nodes store the sum of the values stored by their two children. 

The data structure can be efficiently implemented using an array $A$ of $2n$ value (floating point numbers). 
The position of the $i$th leaf in $A$ is $n+i$. The position of the parent of $i$ is $\lfloor (n + i)/2\rfloor$.

The initialization operation first fills in the last $n$ positions of $A$ by the input values $a_1, \ldots, a_n$. Then it computes the values of all inner nodes and the root in a bottom-up fashion. This takes $O(n)$ time. 

Replacing $a_i$ by $x$ requires a constant-time update for each node from the path from the leaf to the root, thus taking time $O(\log n)$ in total. 

The draw operation can be implemented in time $O(\log n)$ as follows. First, a random number $r$ is drawn from the range $[0, s]$, where $s$ is the value of the root. Then, starting from the root, the left child is selected if $r$ is less or equal to its value $\ell$; otherwise the right child is selected and $r$ is updated to $r - \ell$. The previous step is repeated until a leaf node is reached. The value of the leaf is returned.

\subsection{Proof of Theorem~\ref{thm:pruning}}
\label{se:proof}

We prove the statement in a more general form, namely with
\bes
	\psi(j, S) 
	= \sum_{j \in R \subseteq S} f(R)\, w(R, S)\,,
\ees
only assuming that 
\be \label{eq:unity}
	\sum_{R \subseteq S \subseteq U} w(R, S) \leq 1\quad \textrm{for all $R \subseteq U \subseteq V$}\,.
\ee
Observe that this inequality holds for the function $w(R, S) = (1 + 1/K)^{|R| - K} K^{|R| - |S|}$, the interval sum then being equal to $(1 + 1/K)^{|U| - K} \leq 1$ by the binomial theorem. 

Let $T \subseteq U \subseteq V$. We bound the difference
\be
	\sum_{T \subseteq S \subseteq U} f(S) - \sum_{T \subseteq S \subseteq U} \tilde{f}(S) 
	&\leq& \sum_{T \subseteq S \subseteq U} \min_{j \in S}\big\{\epsilon\cdot \psi(j, S)\big\} \label{eq1}\\
	&\leq&\epsilon \cdot \sum_{T \subseteq S \subseteq U} \sum_{T \subseteq R \subseteq S} f(R)\, w(R, S) \label{eq2}\\
	&=&\epsilon \cdot \sum_{T \subseteq R \subseteq U} \sum_{R \subseteq S \subseteq U} f(R)\, w(R, S) \label{eq3}\\
	&\leq&\epsilon \cdot \sum_{T \subseteq R \subseteq U} f(R) \,, \label{eq4}	
\ee
proving the claim. Justifications for each step above: 
\eqref{eq1} For each set $S$, the difference is $0$ or at most $ \epsilon \cdot \psi(j, S)$ for $j \in S$.
\eqref{eq2} If $T = \emptyset$, the inner sum is without the constraint $j \in S$. Otherwise, $T = \{j\}$ for some $j\in S$.
\eqref{eq3} Switch the order of summation.
\eqref{eq4} Use the assumption \eqref{eq:unity}.

\subsection{Bounding the Posterior Mass Lost due to Pruning}
\label{se:lostmass}

\begin{theorem}\label{thm:lostmass}
Let $\epsilon \in (0, 1)$. For each node $i = 1, \ldots, n$, let $\tilde{\pi}_i$ be an $(\epsilon/n)$-pruning of $\pi_i$. For any DAG on the node set $\{1,\ldots,n\}$ with arc set $A$, let $\pi(A) = \prod_i \pi(A_i)$ and $\tilde{\pi}(A) = \prod_i \hat{\pi}_i(A_i)$. Let $Z = \sum_A \pi(A)$ and $\tilde{Z} = \sum_A \tilde{\pi}(A)$, where the sums are over all DAGs on the $n$ nodes. Then 
\begin{enumerate}[font=\bfseries, label=\emph{(\roman*)},itemindent=1em]
\item $\tilde{Z} \geq (1-\epsilon) Z$; \label{thm:i}
\item The Kullback--Leibler divergence of $\tilde{\pi}/Z$ from $\pi/Z$ is at most $\epsilon/(1-\epsilon)$; \label{thm:ii}
\item The total variation distance between $\pi/Z$ and $\tilde{\pi}/\tilde{Z}$ is at most $\epsilon$. \label{thm:iii}
\end{enumerate}
\end{theorem}

\begin{remark}
Note that the Kullback--Leibler divergence of $\pi/Z$ from $\tilde{\pi}/Z$ is undefined (or $\infty$), unless there is no pruning, i.e., $\pi = \tilde{\pi}$. When $\pi/Z$ is considered a ``ground truth'' this variant is sometimes called the \emph{forward KL divergence}, while the variant in claim (ii) is called the \emph{reverse KL divergence}.
\end{remark}

For a proof, we first extend the approximation guarantee of interval sums to what we call intersection sums.

\begin{proposition}\label{prop:intersectionsums}
For a function $f$ from the subsets of a finite set $V$ to nonnegative real numbers, write 
\bes
	f(T; U) := \sum_{S \subseteq U : S \cap T \neq \emptyset} f(S)\,,
		\quad \textrm{for $T \subseteq U \subseteq V$}\,.
\ees
Suppose $\tilde{f}$ is $\delta$-close to $f$ with $\delta > 0$. Then $\tilde{f}(T; U) \geq (1-\delta) f(T; U)$ for all $T \subseteq U \subseteq V$.
\end{proposition}
\begin{proof}
Label the elements of $T$ arbitrarily as $t_1,\ldots,t_\ell$. Write in terms of interval sums:
\bes
	f(T; U) = \sum_{j=1}^\ell f\big[\{t_j\}, U\setminus \{t_1,\ldots,t_{j-1}\}\big]\,.
\ees
Now the claim follows because $\tilde{f}[T', U'] \geq (1-\delta) f[T', U']$ for each term in the sum.
\end{proof}

Armed with the above result, we turn to prove claim (i). By Equation~3 in \cite{Kuipers2017}, and also Section~3.1 in \cite{Viinikka2020}, we can write 
\bes
	Z = \sum_A \pi(A) = \sum_R \prod_{j=1}^k \prod_{i \in R_i} 
		\sum_{\substack{A_i \subseteq R_1 \cup \cdots \cup R_{j-1}\\ A_i \cap R_{j-1} \neq \emptyset}} \pi_i(A_i)\,,
\ees
where the sum is over all ordered set partitions $R = (R_1, \ldots, R_k)$ of $\{1,\ldots,n\}$. By writing $\tilde{Z}$ correspondingly and using the assumption that each $\tilde{\pi}_i$ is $(\epsilon/n)$-close to $\pi_i$, Proposition~\ref{prop:intersectionsums} gives that each sum of $\tilde{\pi}_i(A_i)$ is at least $1-\epsilon/n$ times the corresponsing sum of $\pi_i(A_i)$. Thus
\bes
	\tilde{Z} \geq (1-\epsilon/n)^n Z \geq (1-\epsilon) Z\,.
\ees

To prove claim (ii), write the Kullback--Leibler divergence as 
\bes
	\sum_{A : \tilde{\pi}(A) > 0} \tilde{\pi}(A)/\tilde{Z} \cdot \ln \frac{\tilde{\pi}(A)/\tilde{Z}}{\pi(A)/Z}
	= \ln \frac{Z}{\tilde{Z}} 
	\leq \ln \frac{1}{1-\epsilon}
	= \ln\Big(1 +\frac{\epsilon}{1-\epsilon}\Big) 
	\leq \frac{\epsilon}{1-\epsilon}\,.
\ees
Here the first equality holds because $\tilde{\pi}(A) = \pi(A)$ when $\tilde{\pi}(A) > 0$. The first inequality follows from part (i) and the second from the fact that $\ln(1 + z) \leq z$ for all $z \geq -1$.

To prove claim (iii), write the total variation distance as
\bes
	\frac{1}{2}\sum_A \Delta(A)\,,
		\quad \textrm{where $\Delta(A) := |\pi(A)/Z - \tilde{\pi}(A)/\tilde{Z}|$}\,.
\ees
Now observe that $\Delta(A) = \pi(A)/Z$ if $\tilde{\pi}(A) = 0$, and $\Delta(A) = (Z/\tilde{Z}-1)\pi(A)/Z$ otherwise. We get 
\bes
	\sum_A \Delta(A) 
	&=& \sum_{A : \tilde{\pi}(A) = 0} \Delta(A) + \sum_{A : \tilde{\pi}(A) > 0} \Delta(A)\\
	&=& \sum_{A : \tilde{\pi}(A) = 0} \pi(A)/Z + (Z/\tilde{Z}-1) \sum_{A : \tilde{\pi}(A) > 0} \pi(A)/Z\\
	&=& 1 - \sum_A \tilde{\pi}(A)/Z + (Z/\tilde{Z}-1) \sum_A \tilde{\pi}(A)/Z\\
	&=& 2 (1 - \tilde{Z}/Z)\\
	&\leq& 2 \epsilon\,,	 
\ees
where the last inequality follows from claim (i). This completes the proof.

\subsection{Exact Sampling and Summing over DAGs}
\label{se:exact}

The exact arc posterior probabilities can be computed in time $O(3^n n)$ by the algorithm of \citet{Tian2009}. An efficient implementation of this algorithm is given by \citet{Pensar2020}, who also extend it to compute, within the same complexity, the ancestor relation probabilities for all pairs of nodes. We used this implementation from the \emph{BIDA} package (released under the MIT license, \url{https://github.com/jopensar/BIDA}) in the present work for datasets on 20 or fewer variables.

There are also exponential-time algorithms for exact sampling of DAGs from modular distributions \citep{Harviainen2024,Talvitie2019uai}. With $t$ independently sampled DAGs, one can estimate the probability of any structural feature to within an absolute error of $O\big(1/\!\sqrt{t}\big)$ with high probability, thanks to basic properties of Monte Carlo integration. These sampling algorithms can be particularly useful in (rare) cases where the exact algorithms encounter numerical challenges, such as catastrophic cancellations. 
We faced such a case with the dataset generated from the \child{} BN and used the implementation of \citet{Harviainen2024} (released in the supplementary material of the article, with no specified license) to estimate the arc posterior probabilities using a sample of 10,{}000 DAGs.

\section{BAYESIAN DAG SAMPLERS}
\label{se:samplers}

\begin{table*}[t!]
\caption{Recent Bayesian DAG samplers}
\label{table:samplers}
\medskip
\centering
{\small
\begin{tabular}{p{120pt}p{130pt}p{200pt}}
\toprule
Name \& reference & Approach & Notes \\
\midrule
\addlinespace[4pt]
\gibby{} \newline \emph{this paper} & MCMC in the DAG space & New algorithmic techniques to expedite both basic moves and moves requiring sums over parent sets. \\
\addlinespace[4pt]
\bidag{} \newline \citet{Kuipers2022} & MCMC in the space of ordered node partitions & Fast moves between partitions using a lookup table for score sums, assuming a small set of candidate parents.\\
\addlinespace[4pt]
\daggflow{} \newline \citet{Deleu2022} & Independent sampling using generative flows & Assumes a GPU is available. \\
\addlinespace[4pt]
\cmidrule(l{12pt}r{12pt}){1-3}
\addlinespace[4pt]
\emph{BayesDAG} \newline \citet{Annadani2023} &  Stochastic gradient MCMC and variational inference in the joint structure--parameter space & Assumes continuous data and an additive-noise model. Does not apply to categorical data. \\
\addlinespace[4pt]
\emph{BCD Nets} \newline \citet{Cundy2021} & Variational inference with a limited class of approximating distributions & Assumes continuous data and a linear model. \\ 
\addlinespace[4pt]
\emph{Gadget} \newline \citet{Viinikka2020} & MCMC in the space of ordered node partitions & A more scalable implementation 
of \bidag{}. Poor accuracy reported \citep{Deleu2022}. Code apparently not available at the moment.\\ 
\addlinespace[4pt]
\emph{Layering-MCMC} \newline \citet{Viinikka2020uai} & MCMC in the space of layers of ordered node partitions & The implementation not made to be scalable to networks on dozens of nodes. Suffers from similar mixing problems than partition MCMC that does not use parent set resampling moves. For \zoo{} MAD values around $0.05$ reported \citep[Fig.~4]{Viinikka2020uai}.\\
\addlinespace[4pt]
\emph{Gibbs} \newline \citet{Goudie2016} & MCMC in the DAG space & Slow for larger max-indegree. Could be improved by the pruning technique presented in this paper. \\
\addlinespace[4pt]
\emph{BEANDisco} \newline \citet{Niinimaki2016} & MCMC and annealed importance sampling in the space of bucket orderings of nodes & Non-uniform structure prior over equivalent DAGs.\\
\addlinespace[4pt]
\bottomrule
\end{tabular}
}
\end{table*}

Table~\ref{table:samplers} gives a summary of some recent methods for Bayesian sampling of DAGs. The first three methods are empirically compared in this paper. The rest are included to highlight related methods we did not include in our empirical study for varying reasons: (i) the method or its available implementation is not applicable to categorical data or does not support score equivalent structure priors (\emph{BayesDAG}, \emph{BCD Nets}, \emph{BEANDisco}), (ii) the method takes the same approach as \gibby{} and, therefore, can be improved by the techniques presented in this paper (\emph{Gibbs}), or (iii) the method takes the same approach as one of the tested methods and the available implementation either is not designed to be scalable or appears to suffer some technical issues (\emph{Layering-MCMC}, \emph{Gadget}).

\paragraph{\daggflow{}.} We used \daggflow{} (released under the MIT license) with its default settings: 
We ran the optimization phase with 100{,}000 iterations and a learning rate of $10^{-5}$ and then sampled 1{,}000 DAGs.
We were not able to run \daggflow{} on datasets over 17 or more variables: the process froze and made no progress even after 12 hours; see Section~\ref{se:environment} for technical information about our computing environment. We tested using a dataset of size 1,{}000 generated from the \child{} benchmark BN ($n = 20)$, the \zoo{} dataset of size 101 ($n=17$), and the \emph{Cervical Cancer Behavior Risk} dataset of size 72 ($n = 20$) from the UCI Machine Learning Repository~\mbox{(\url{https://doi.org/10.24432/C5402W})}. 

\paragraph{\bidag{}.} We used \bidag{} (released under the GPL-2 and GPL-3 licenses) with its default settings, as follows. 
For the small networks, we aimed at running times ranging from 20 to 40 minutes. To match this range, we let the sampler perform 
$6 \cdot 10^6$ iterations for \asia{}, $7\cdot 10^6$ iterations for \sachs{} and \zoo{}, and $4\cdot 10^6$ iterations for \child{}. We set the thinning parameter to reduce the sampled DAGs to $1\cdot 10^5$.
For the larger networks we aimed at running times ranging from 10 to 15 hours. Consequently, we let the sampler perform $1\cdot 10^8$ iterations for \hailfinder{} and $4\cdot 10^7$ iterations for \andes{}. The thinning parameter was set to $10^4$ for both datasets. Despite the lower number of iterations for \andes{}, we used the same thinning factor to ensure that DAG sampling does not dominate the total time requirement (moves between ordered partitions being fast as they do not require computationally expensive DAG sampling).
In all runs, the iterative order sampling procedure was used to find a selections of candidate parents for each node. 

\paragraph{\gibby{}.} We used \gibby{} in the sampling accuracy studies (Section~\ref{se:comparison}) as follows.  
We tuned the parameters to match the running time of \bidag{} and to output the same number of DAGs (except for \andes{}, for which we generated $10^4$ DAGs). 
For the small networks, we let \gibby{} take 100 basic moves (i.e., steps) and 2 REV moves per 1 MBR move.  
For the larger networks, the MBR move proved particularly useful and we let \gibby{} take 50 basic moves and 1 REV move per 1 MBR move. 
In all runs, we set the accuracy parameter $\epsilon$ to $2^{-15}$ and started from the empty DAG. 

\section{COMPUTING ENVIRONMENT}\label{se:coenv}
\label{se:environment}

We have implemented \gibby{} in \cpp{} and compiled using \emph{GCC} with the flags \texttt{-march=native -O3}. The source code is available at \url{https://github.com/sums-of-products/gibby}.

\gibby{} and \bidag{} were run as single-thread processes on a cluster with AMD EPYC 7452 processors with 32 cores, running at 2,{}350 MHz, and with 504 GB of RAM.

\daggflow{} was run on a machine with an NVIDIA GeForce RTX 4070 card (5,{}888 CUDA cores).

\section{DATA GENERATION AND PREPROCESSING}

The benchmark BNs we used are from Scutari's Bayesian Network Repository \citep{ScutariBNrepository} and  licensed under CC BY-SA 3.0.

The datasets we used are from the UCI Machine Learning Repository \citep{UCIrepo} and licensed under CC~BY~4.0. 

From the benchmark Bayesian networks, data were generated in a standard way by forward sampling. We only included in our tests one generated dataset of each size. We acknowledge that independently generated replicas of the datasets could yield slightly different numerical results.  

No further preprocessing of the data was done. (All variables are discrete, no discretization was needed. There are no missing values. We kept all variables.)

\section{EMPIRICAL RESULTS}

This section reports some additional empirical results. 

\subsection{\gibby{} versus \GC{}}
\label{se:gibbymore}

Table~\ref{table:versusgcmore} shows the speed of \GC{} and \gibby{} on datasets of size 2,{}000 and 500. Both algorithms are faster on the larger datasets when measured by the number of steps taken in a second. This is explained by the smaller acceptance ratio, which is associated with more steps corresponding to rejected proposals and thus less computation required for updating data structures. 

However, the largest of the networks, \pigs, shows a different pattern: The tentative acceptance rate $b$ is lower for the smaller dataset, which helps \gibby{}. On the contrary, \GC{} becomes somewhat slower on the smaller dataset, which might be explained by the larger acceptance rate $a$ (\pigs{} being an exception, however).

\begin{table*}[t!]
\caption{Speed of \ourbasic{} vs.\ \emph{GC} on datasets sampled from benchmark Bayesian networks}
\label{table:versusgcmore}
\medskip
\centering
{\small
\begin{tabular}{rcccccccc}
\toprule
   &       &      &       & \multicolumn{3}{c}{Steps per microsecond} & \multicolumn{2}{c}{Acceptance ratio (\%)}  \\
\cmidrule(l{4pt}r{4pt}){5-7} \cmidrule(l{4pt}r{4pt}){8-9}
Network & Nodes & Arcs & Max-indeg & \GC{} & \ourbasic{}$^{\dagger}$ & \ourbasic{}$^{\ddagger}$ & $a$ & $b$ \\
\midrule 
\multicolumn{9}{c}{Dataset size 2,{}000}\\
\addlinespace[2pt]
\alarm	&  37 &  46 & 4 & 3.6 & {\bf76} & 51 & 0.35 & 4.4  \\
\hailfinder	&  56 &  66 & 4 & 3.2 & {\bf106} & 92 & 0.24 & 1.3  \\
\pathfinder	& 109 & 195 & 5 & 3.0 & {\bf31} & 30 & 0.28  & 1.5 \\
\andes 		& 233 & 338 & 6 & 1.9 & 59 & {\bf65} & 0.038 & 0.21  \\
\pigs 		& 441 & 592 & 2 & 3.1 & 617 & {\bf628} & 0.014  & 0.14  \\
\addlinespace[2pt]
\cmidrule(l{12pt}r{12pt}){1-9} 
\multicolumn{9}{c}{Dataset size 500}\\
\addlinespace[2pt]
\alarm	&  37 &  46 & 4 & 3.3 & {\bf58} & 46 & 0.61 & 3.5  \\
\hailfinder	&  56 &  66 & 4 & 3.0 & {\bf70} & 68 & 0.58  & 1.1 \\
\pathfinder	& 109 & 195 & 5 & 1.7 & {\bf22} & 21 & 0.40 & 1.2  \\
\andes 		& 233 & 338 & 6 & 1.5 & 31 & {\bf36} & 0.075 & 0.15  \\
\pigs 		& 441 & 592 & 2 & 2.2 & {\bf1068} & 984 & 0.0083 & 0.064   \\
\addlinespace[2pt]
\bottomrule
\multicolumn{5}{l}{\scriptsize $^{\dagger}$Constant-time acyclicity checks by maintaining the ancestor relation} &
\multicolumn{3}{l}{\scriptsize $^{\ddagger}$Acyclicity checks by path-finding algorithms}
\end{tabular}
}
\end{table*}

\subsection{Efficiency of Score Pruning}
\label{se:pruningmore}

Table~\ref{table:pruningmore} shows the efficiency of score pruning on datasets of size 2,{}000 and 500. We observe that the data size has a clear effect on pruning efficiency, the larger size either lowering or increasing the fraciton of kept parent sets. 
For the \hailfinder{} and \pigs{} datasets, the savings are lower on the smaller size (500), while for the other cases pruning is more efficient for the larger size (2,{}000). 

\begin{table*}[t!]
\caption{Efficiency of pruning on datasets sampled from benchmark Bayesian networks} 
\label{table:pruningmore}
\medskip
\centering
{\small
\begin{tabular}{rcccw{c}{55pt}w{c}{30pt}w{c}{30pt}w{c}{55pt}w{c}{30pt}w{c}{30pt}}
\toprule
 & & & & \multicolumn{3}{c}{$\epsilon = 2^{-15}$} & \multicolumn{3}{c}{$\epsilon = 2^{-10}$} \\
\cmidrule(l{4pt}r{4pt}){5-7} \cmidrule(l{4pt}r{4pt}){8-10}
Network & Nodes & \!\!\!Max-indeg & $K$$^{\dagger}$ & Kept (\%) & \TimeB$^{*}$ & \TimeC & Kept (\%) & \TimeB & \TimeC \\
\midrule 
\multicolumn{10}{c}{Dataset size 2000}\\
\addlinespace[2pt]
\alarm		&  37 &  4 & 36 & 4.0 & 9.4 & 17 & 3.0 & 8.7 & 17 \\
\hailfinder	&  56 &  4 & 55 & 0.14 & 9.3 & 151 & 0.13 & 9.1 & 150 \\
\pathfinder	& 109 &  5 & 64 & 14 & 9930 & 24200 & 9.3 & 8540 & 24900 \\
\andes		& 223 &  6 & 64 & 0.055 & 5470  & -- & 0.017 & 1980  & -- \\
\pigs		& 441 &  2 & 64 & 35 & 13 & 13 & 34 & 13 & 13 \\
\pigs		& 441 &  \hspace{4pt}4$^{\ddagger}$ & 64 & 0.35 & 211 & 6810 & 0.34 & 206 & 6980 \\
\addlinespace[2pt]
\cmidrule(l{12pt}r{12pt}){1-10} 
\multicolumn{10}{c}{Dataset size 500}\\
\addlinespace[2pt]
\alarm		&  37 &  4 & 36 & 8.8 & 7.6 & 11 & 4.9 & 6.7 & 11 \\
\hailfinder	&  56 &  4 & 55 & 0.067 & 5.1 & 98 & 0.064 & 5.0 & 98 \\
\pathfinder	& 109 &  5 & 64 & 28 & 9004 & 16627 & 19 & 7829 & 14563 \\
\andes		& 223 &  6 & 64 & 0.46 & 21711  & -- & 0.10 & 7154  & -- \\
\pigs		& 441 &  2 & 64 & 14 & 8.7 & 8.7 & 13 & 8.7 & 8.7 \\
\pigs		& 441 &  \hspace{4pt}4$^{\ddagger}$ & 64 & 0.076 & 73 & 4463 & 0.070 & 65 & 4380 \\
\addlinespace[2pt]
\bottomrule
\multicolumn{10}{l}{\scriptsize 
$^{\dagger}$No.\ candidate parents per node\;
$^{*}$No.\ seconds by bottom-up (\TimeB) and complete pruning (\TimeC)\;
$^{\ddagger}$Larger than in the generating network
}

\end{tabular}
}
\end{table*}

\subsection{Accuracy of Posterior Sampling}
\label{se:accuracymore}

For the experiments, we had to set the max-indegree parameter, which gives an upper bound for the number of parents of any node.
When viewed as part of the structure prior, the value of the parameter affects the (true) posterior distribution. Since we were able to deal with the exact posterior distributions for the small networks ($n \leq 20$), we set the max-indegree as large a value as possible. 
We summarize these choices in Table~\ref{table:maxindegree}.

\begin{table*}[t!]
\caption{The values of the max-indegree parameter used in the experiments}
\label{table:maxindegree}
\medskip
\centering
{\small
\begin{tabular}{rccp{210pt}}
\toprule
Name & Nodes & Max-indeg & Explanation \\
\midrule
\addlinespace[2pt] 
\multicolumn{4}{c}{Small networks}\\
\addlinespace[2pt]
\asia & 8 & 7 & The largest possible\\
\addlinespace[2pt]
\sachs & 11 & 10 & The largest possible \\
\addlinespace[2pt]
\child & 20 & 9 & The total number of possible parent sets to not exceed $10^7$ \\
\addlinespace[2pt]
\zoo{} & 17 & 13 & The maximum number of parents \bidag{} can handle (the \texttt{hardlimit} parameter)\\
\addlinespace[2pt]
\cmidrule(l{12pt}r{12pt}){1-4} 
\multicolumn{4}{c}{Large networks}\\
\addlinespace[2pt]
\hailfinder & 56 & 4 & The max-indegree of the data-generating DAG\\
\addlinespace[2pt]
\andes & 233 & 6 & The max-indegree of the data-generating DAG\\
\addlinespace[4pt]
\bottomrule
\end{tabular}
}
\end{table*}

\begin{figure*}[t!]
\begin{center}
{\small
\begin{tabular}{ccc}
\multicolumn{3}{c}{\asia{} ($n = 8$)} \\
\midrule
\addlinespace[4pt]
\gibby{} & \bidag{} & \daggflow{} \\
\hspace{-10pt}
\includegraphics[height=0.23\textwidth]{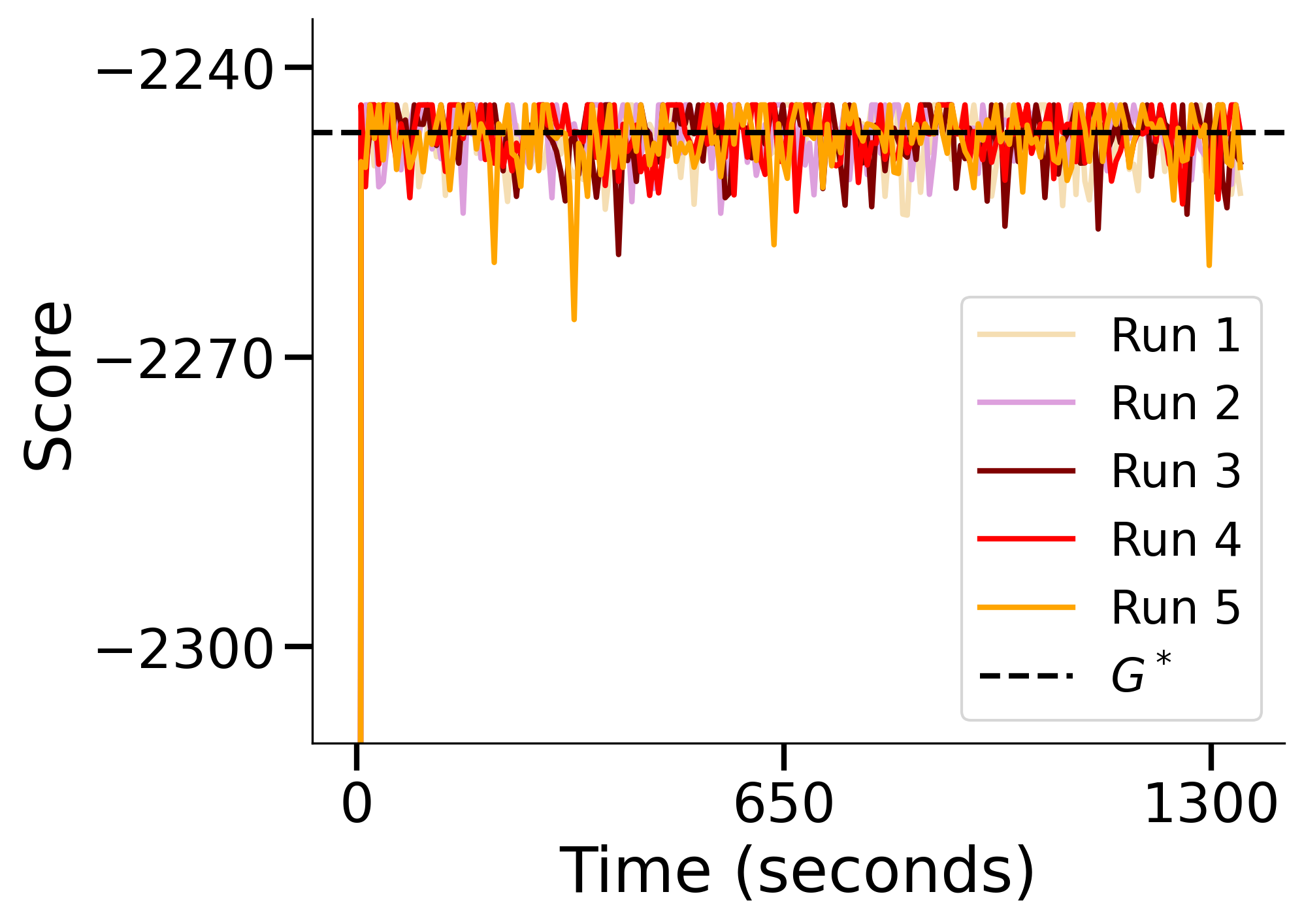} &  
\hspace{-10pt}
\includegraphics[height=0.23\textwidth]{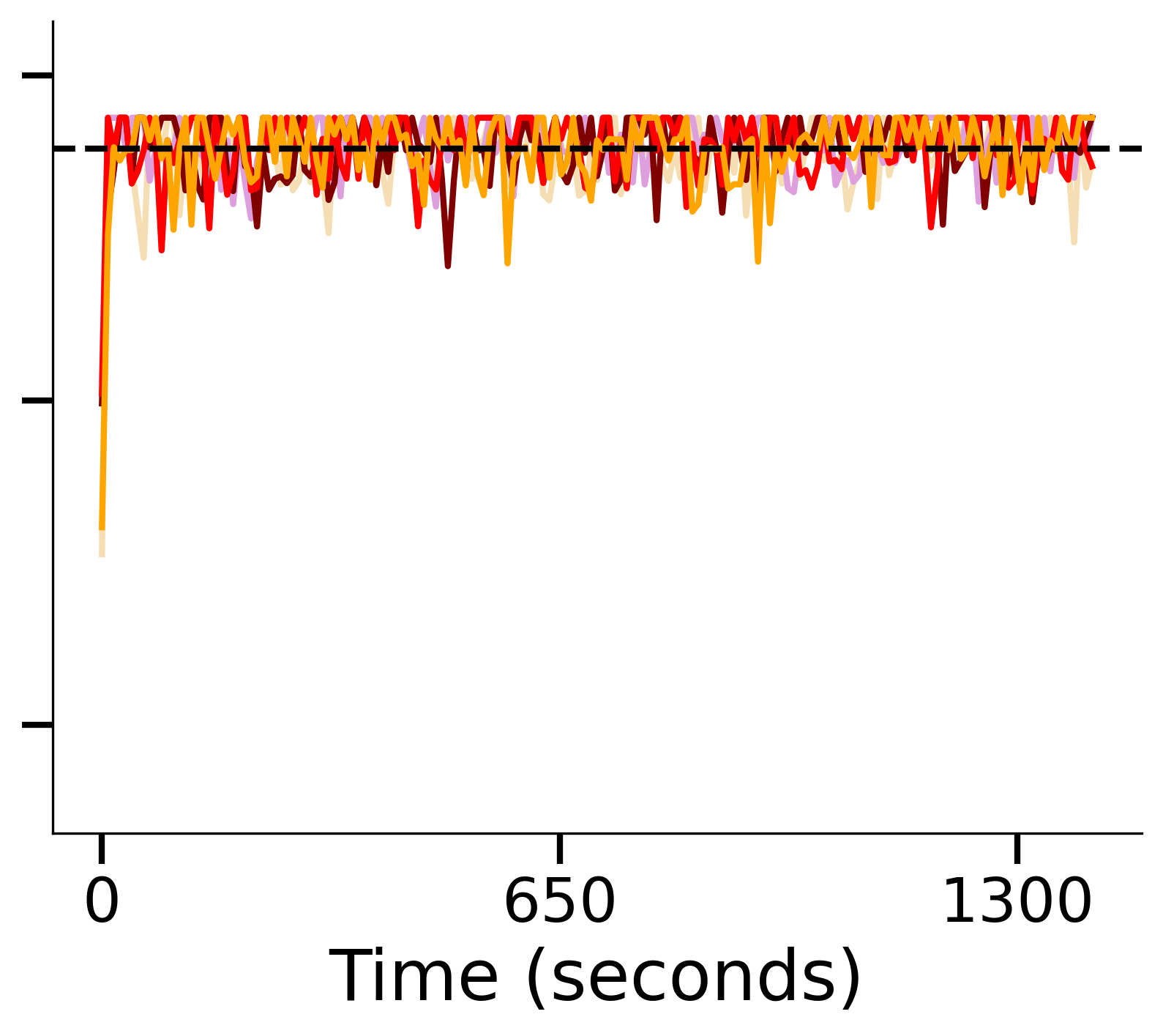} & 
\hspace{-10pt}
\includegraphics[height=0.23\textwidth]{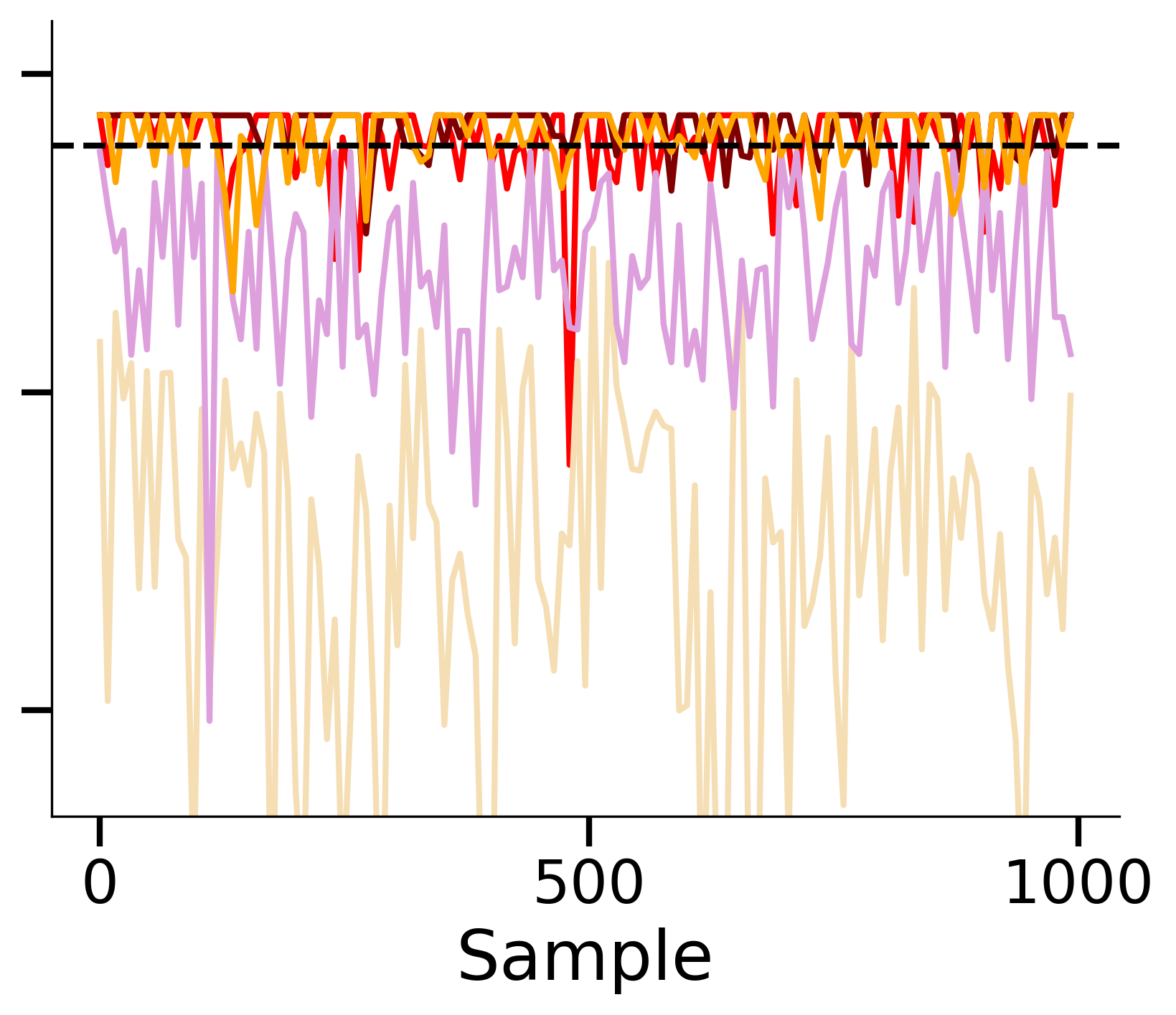} \\ 
\addlinespace[4pt]
\includegraphics[height=0.23\textwidth]{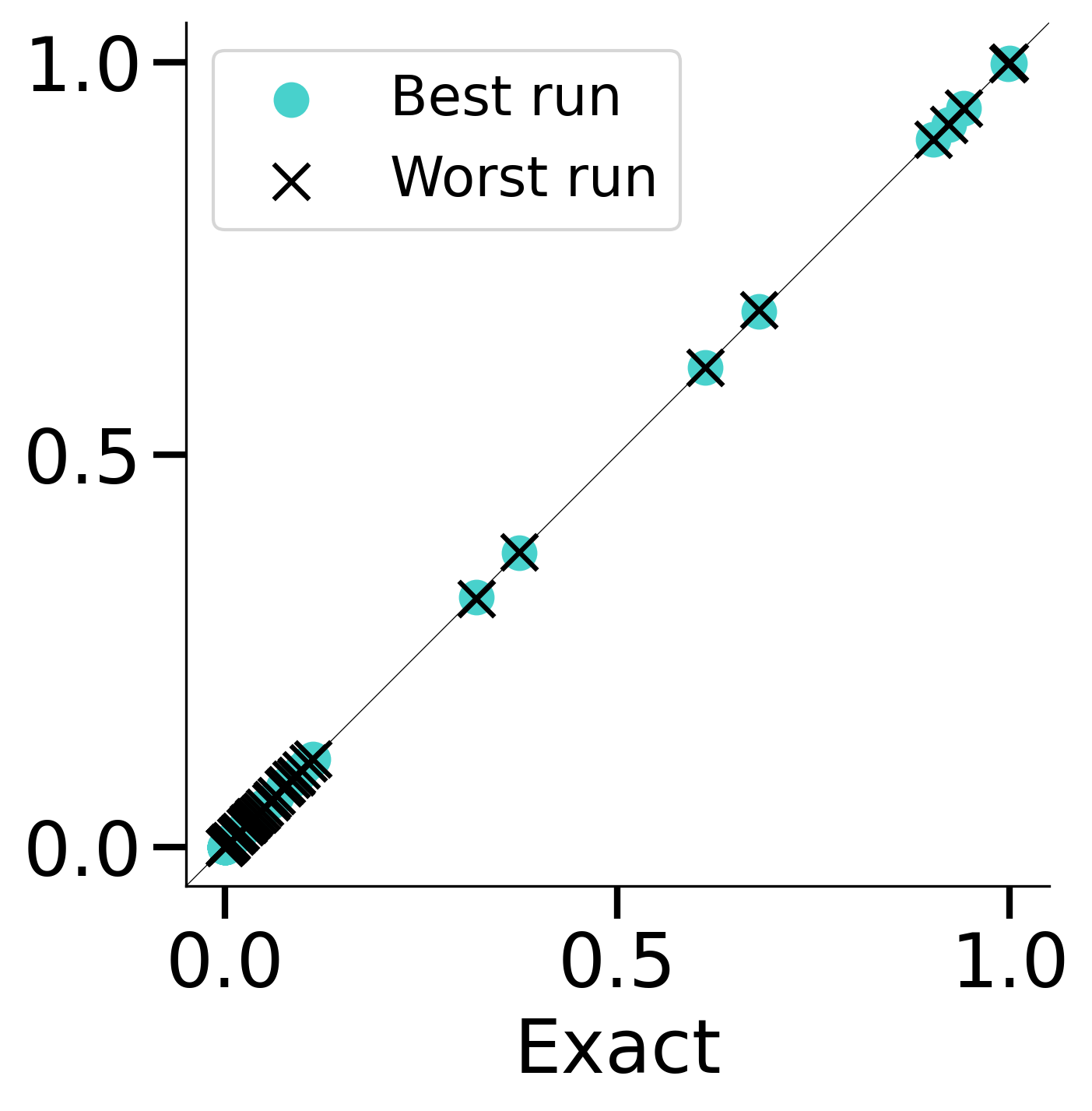} & 
\hspace{-20pt}
\includegraphics[height=0.23\textwidth]{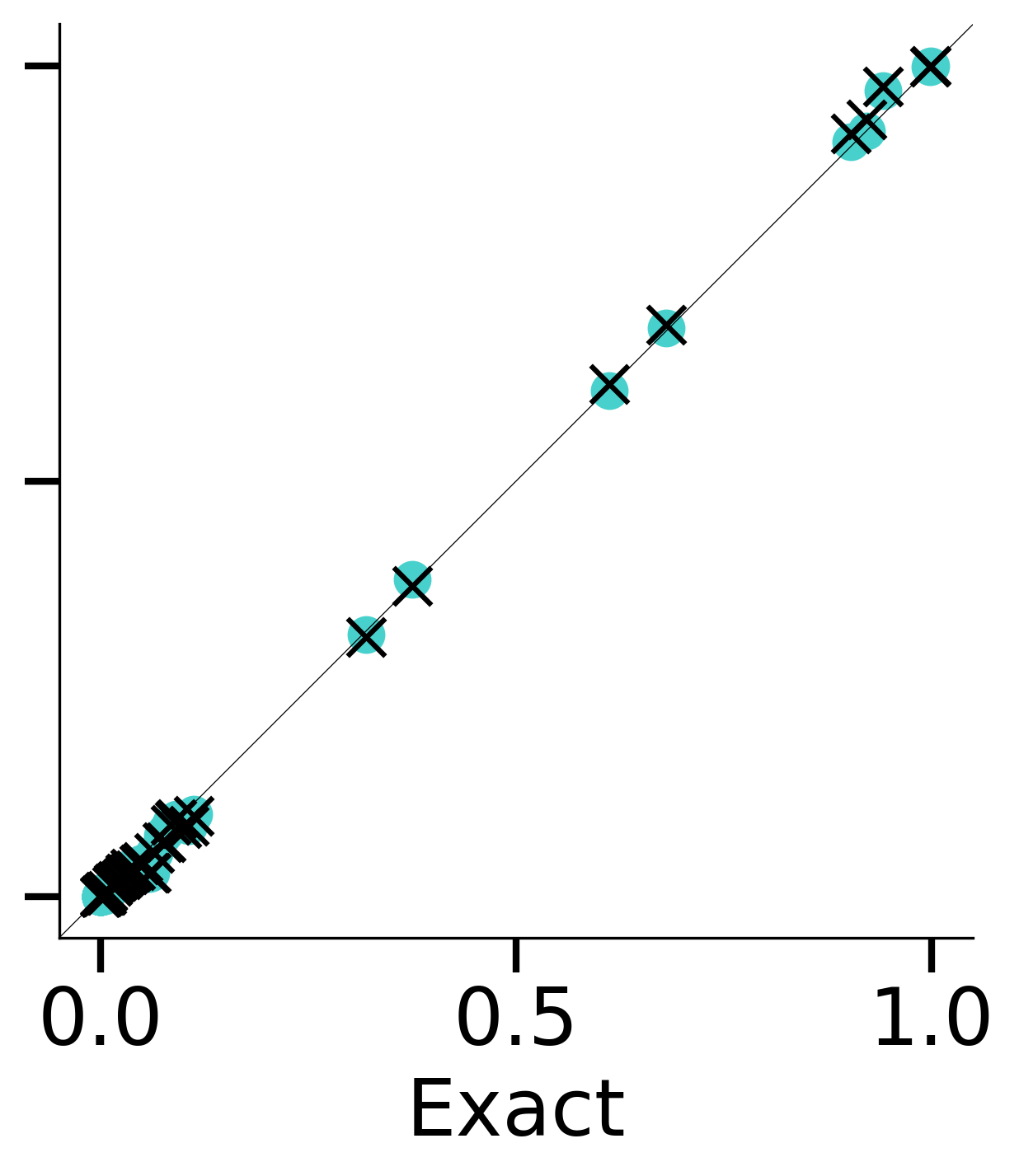} & 
\hspace{-20pt}
\includegraphics[height=0.23\textwidth]{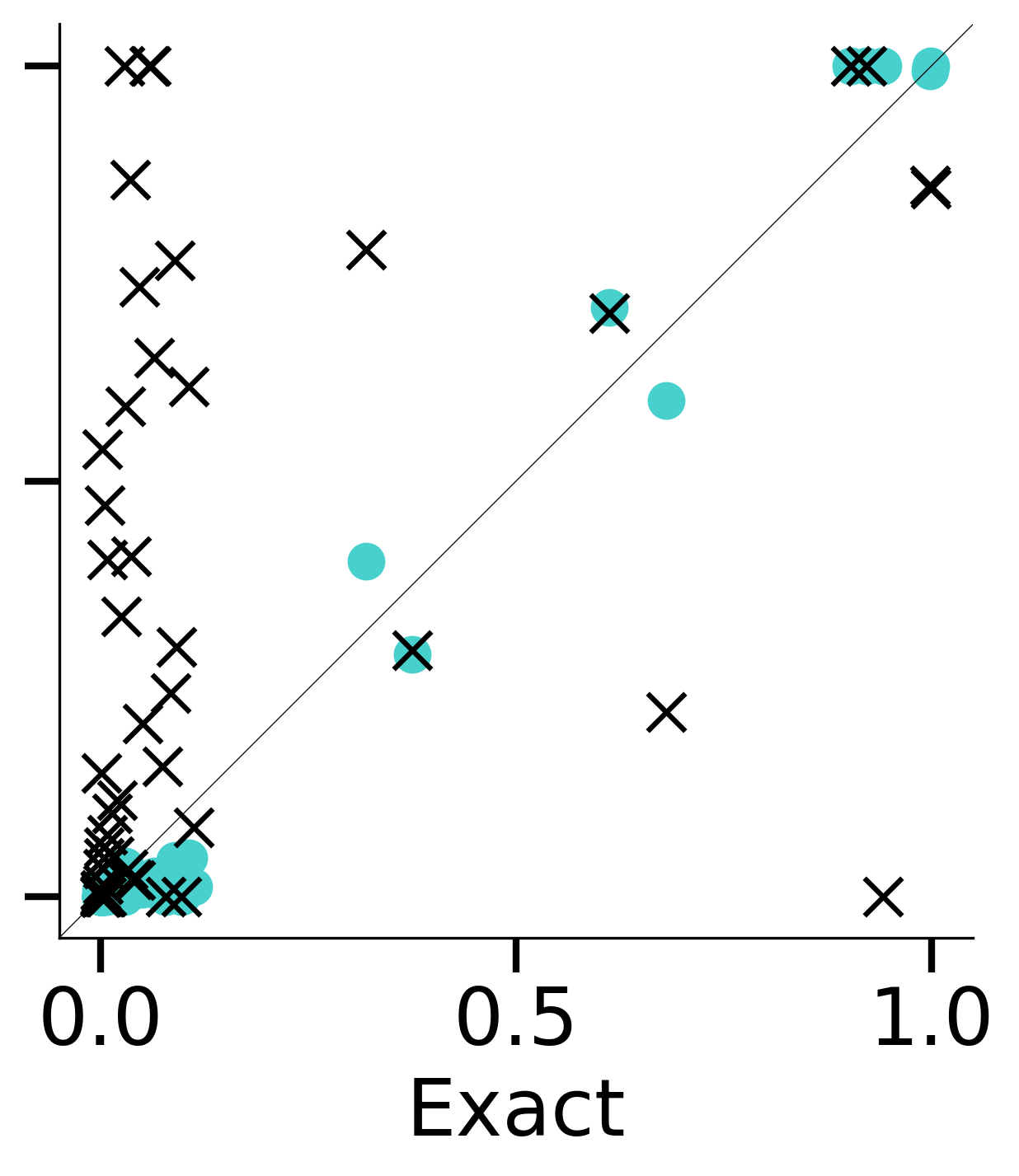} \\ 
\end{tabular}
}
\caption{\label{fig:traceasia}
Performance of DAG samplers on an \asia{} dataset of size 1,{}000.
\emph{Top:} The log posterior probability of the sampled DAG in five independent runs. The posterior probability of the data-generating DAG $G^*$ is marked by a horizontal black line.
\emph{Bottom:} The estimated arc posterior probabilities at the end of the runs, against the exact values. For each node pair, the best and the worst estimate over the five runs are shown.
}
\end{center}
\end{figure*}

\begin{figure*}[t!]
\begin{center}
{\small
\begin{tabular}{ccc}
\multicolumn{3}{c}{\sachs{} ($n = 11$)} \\
\midrule
\addlinespace[4pt]
\gibby{} & \bidag{} & \daggflow{} \\
\hspace{-10pt}
\includegraphics[height=0.23\textwidth]{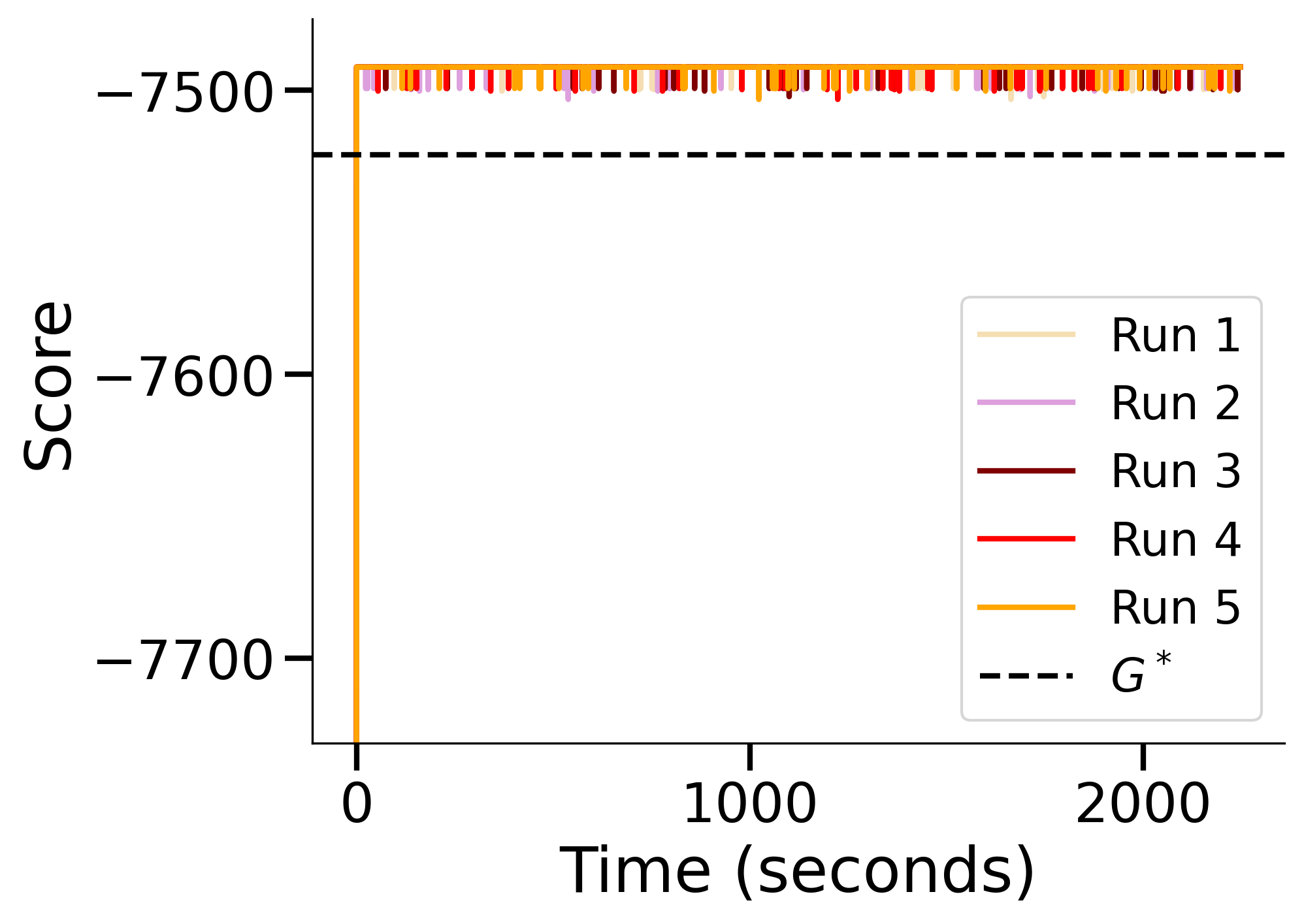} &  
\hspace{-10pt}
\includegraphics[height=0.23\textwidth]{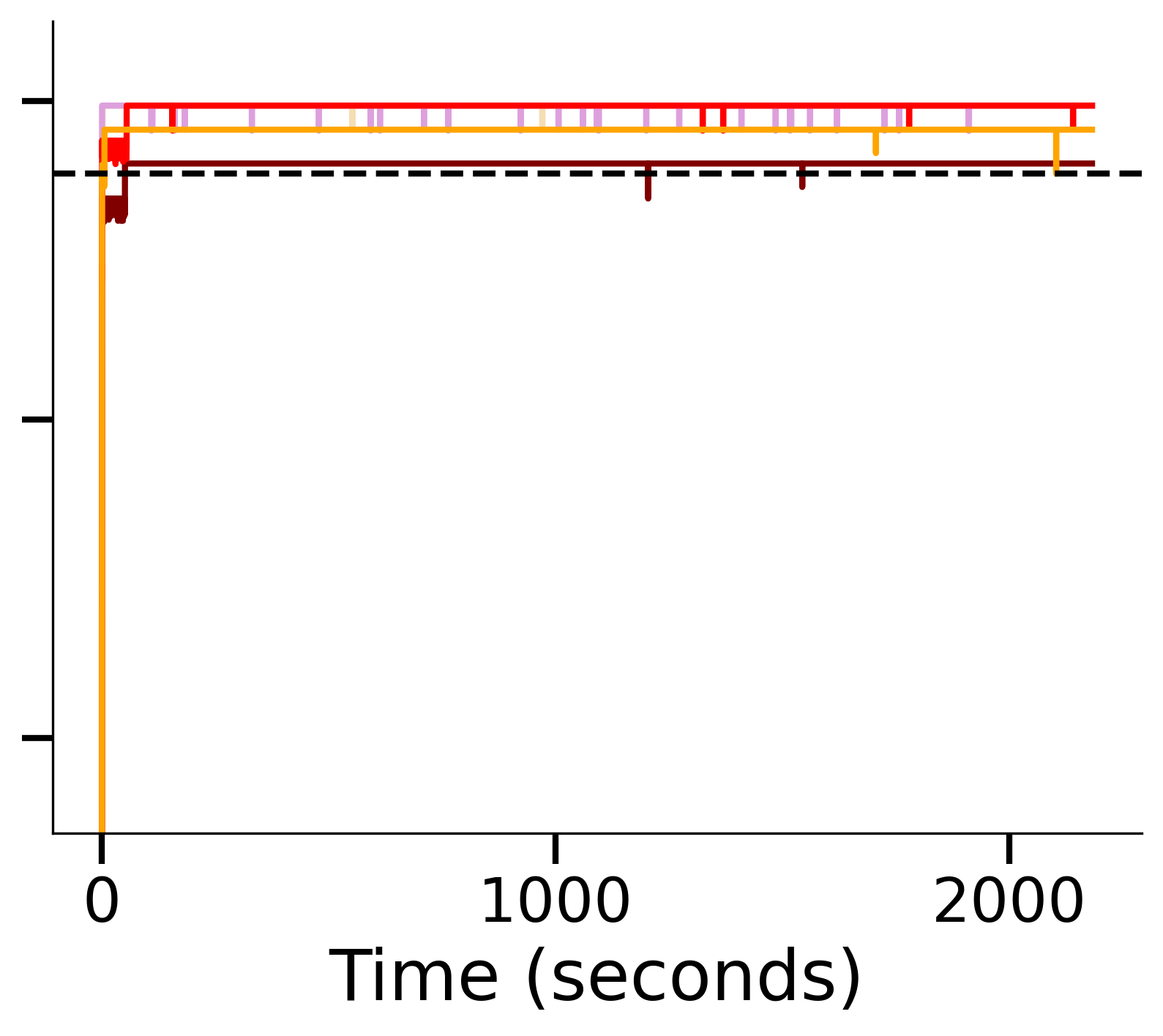} & 
\hspace{-10pt}
\includegraphics[height=0.23\textwidth]{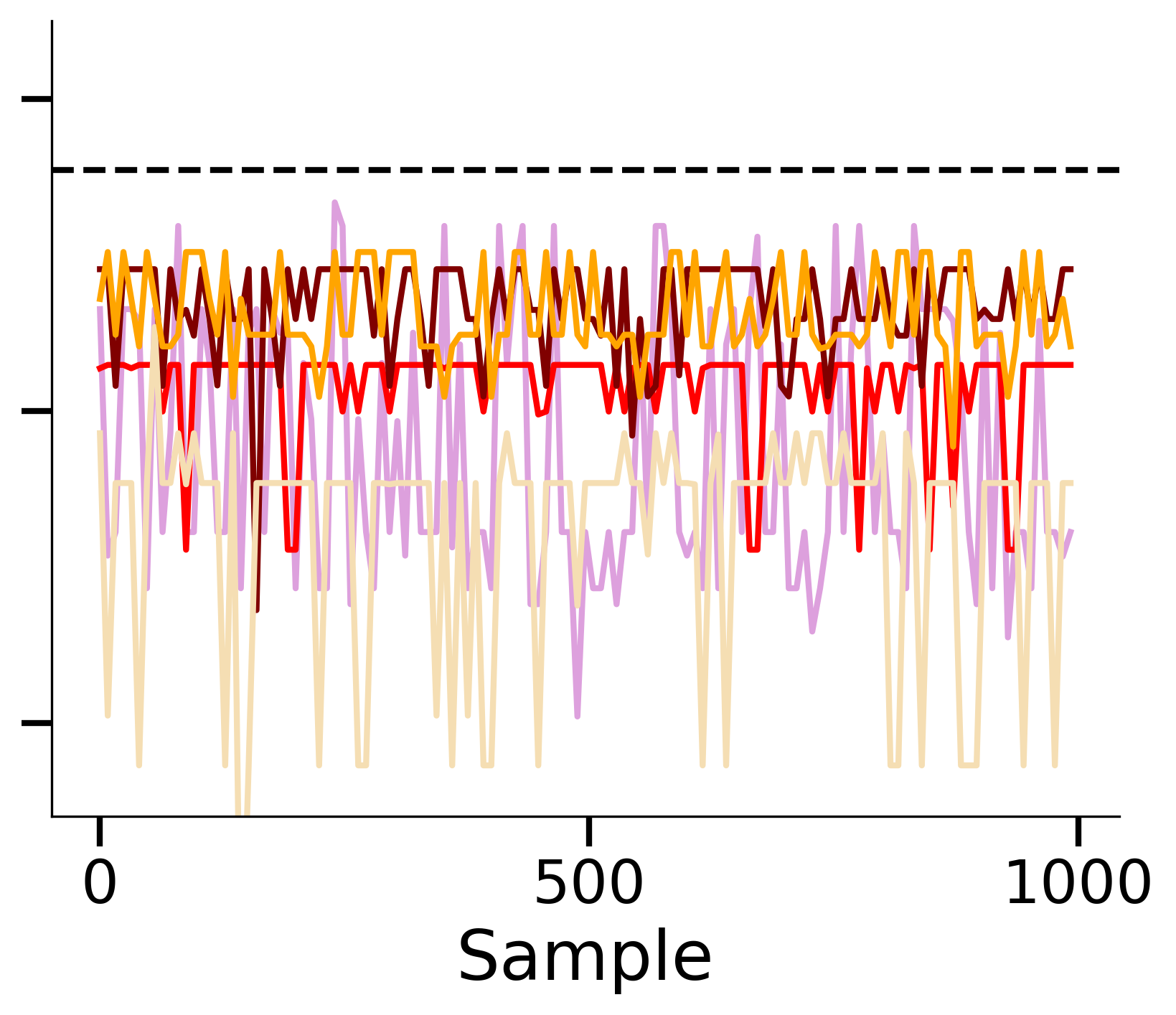} \\ 
\addlinespace[4pt]
\includegraphics[height=0.23\textwidth]{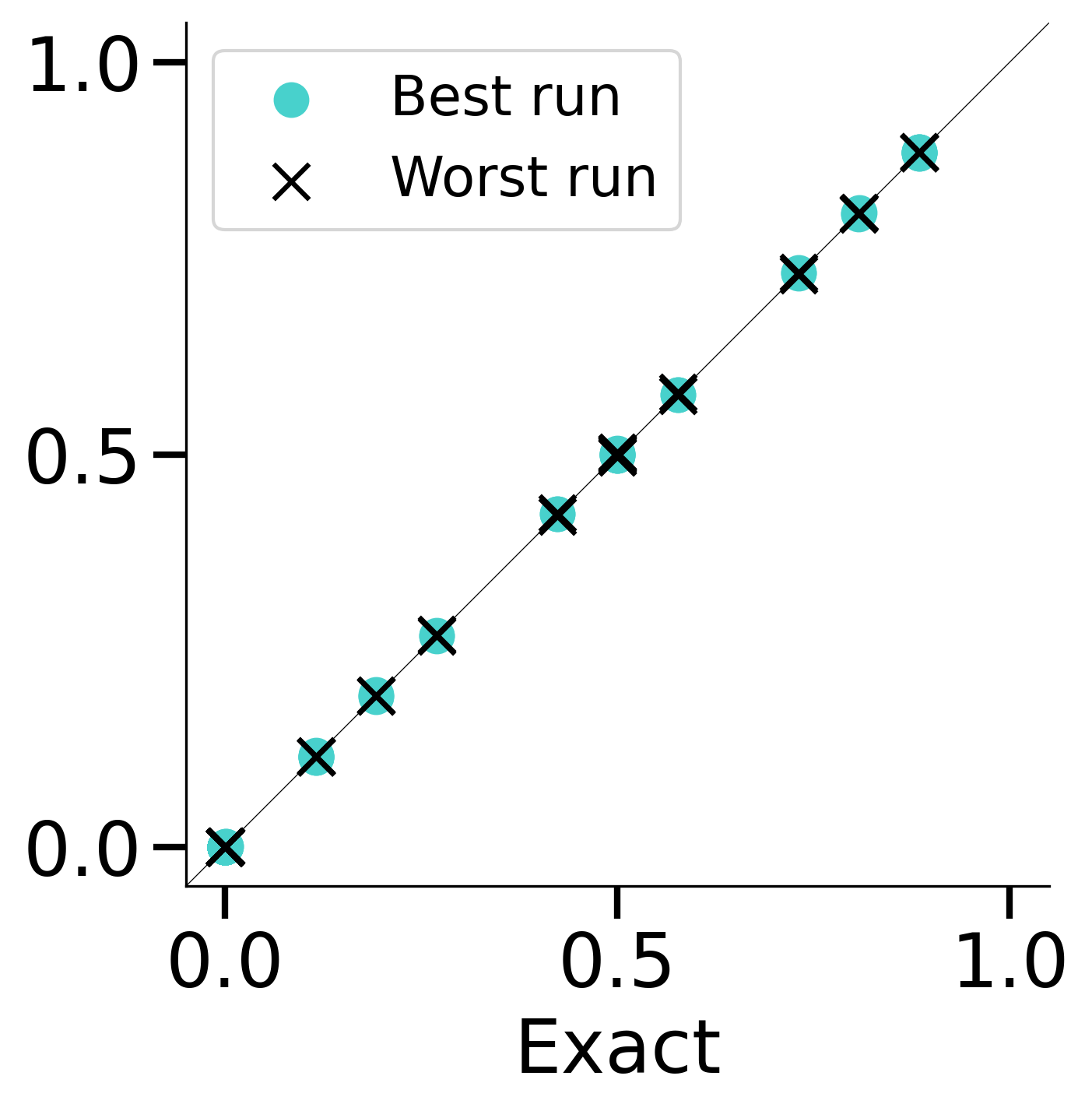} & 
\hspace{-20pt}
\includegraphics[height=0.23\textwidth]{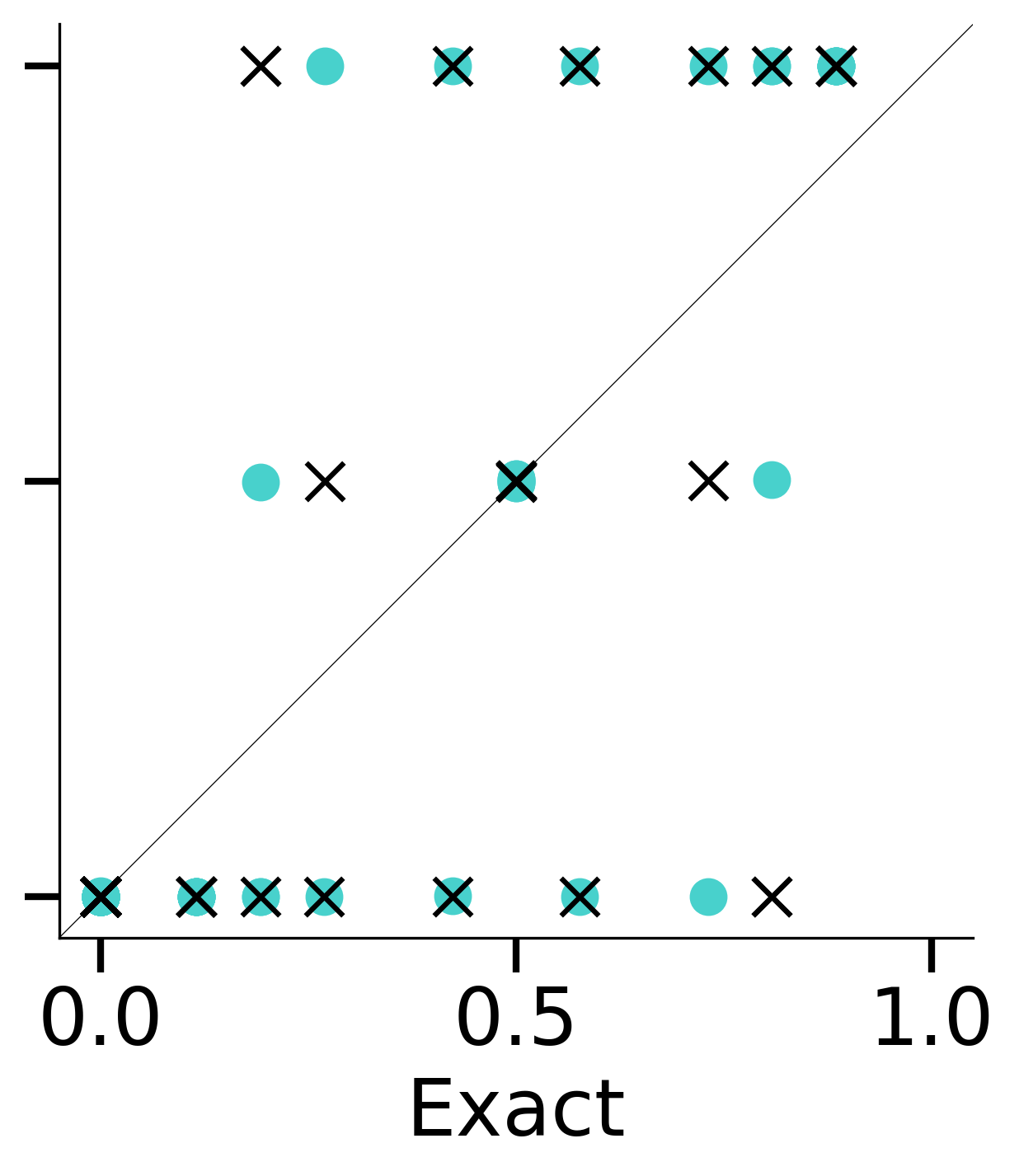} & 
\hspace{-20pt}
\includegraphics[height=0.23\textwidth]{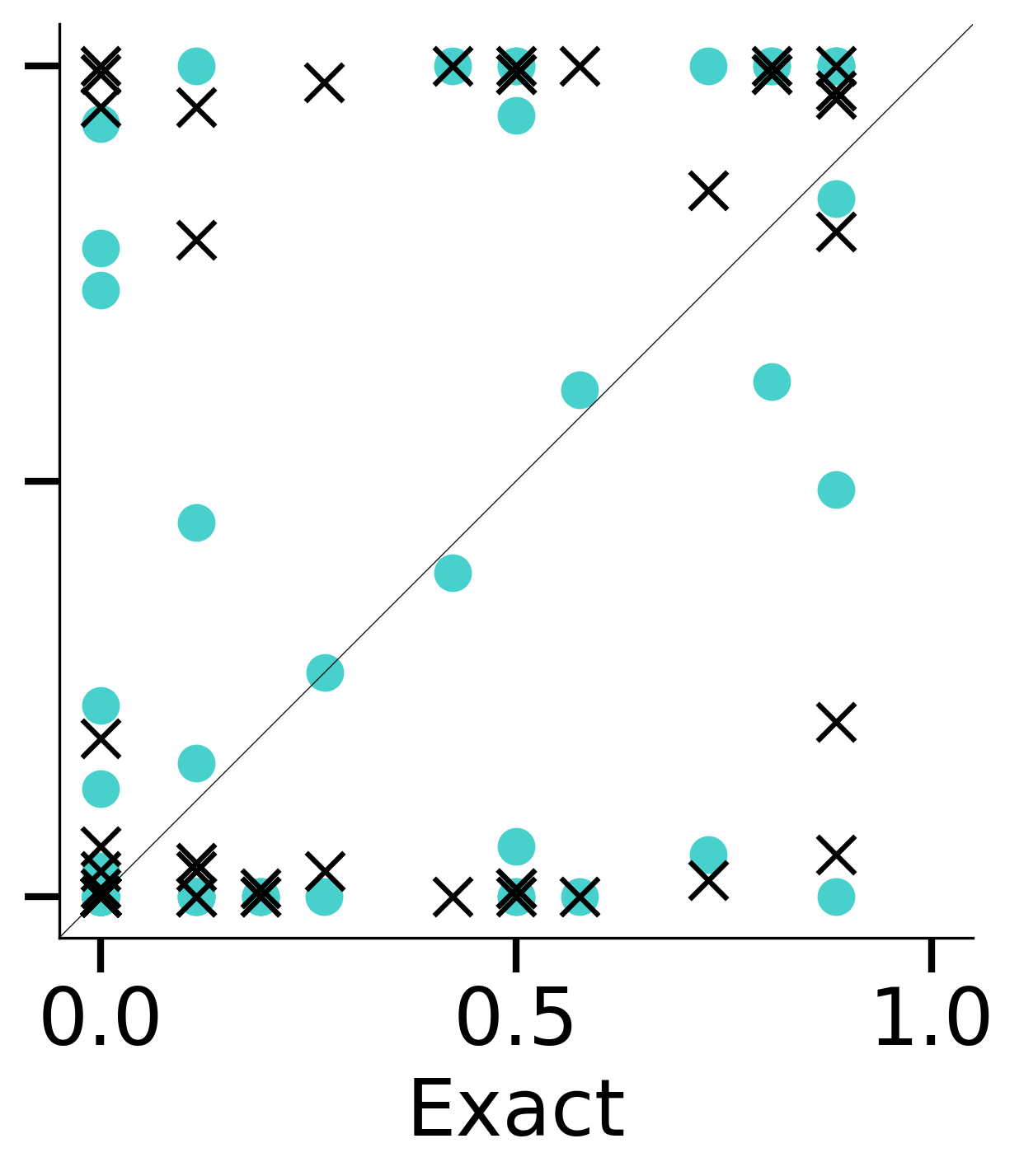} \\ 
\end{tabular}
}
\caption{\label{fig:tracesachs}
Performance of DAG samplers on an \sachs{} dataset of size 1,{}000.
\emph{Top:} The log posterior probability of the sampled DAG in five independent runs. The posterior probability of the data-generating DAG $G^*$ is marked by a horizontal black line.
\emph{Bottom:} The estimated arc posterior probabilities at the end of the runs, against the exact values. For each node pair, the best and the worst estimate over the five runs are shown.
}
\end{center}
\end{figure*}

\begin{figure*}[t!]
\begin{center}
{\small
\begin{tabular}{cc}
\multicolumn{2}{c}{\child{} ($n = 20$)} \\
\midrule
\addlinespace[4pt]
\gibby{} & \bidag{} \\
\hspace{-10pt}
\includegraphics[height=0.23\textwidth]{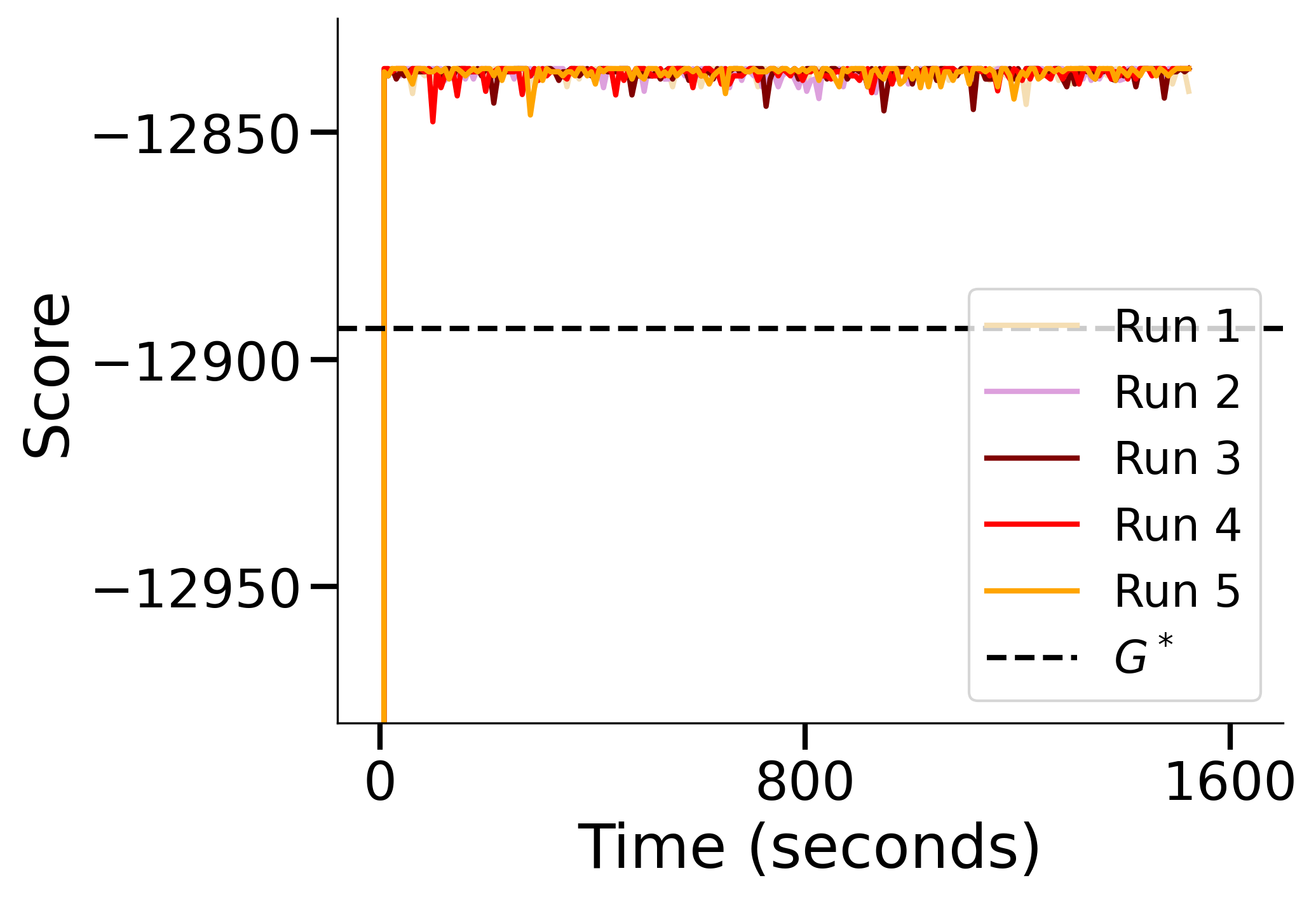} &  
\includegraphics[height=0.23\textwidth]{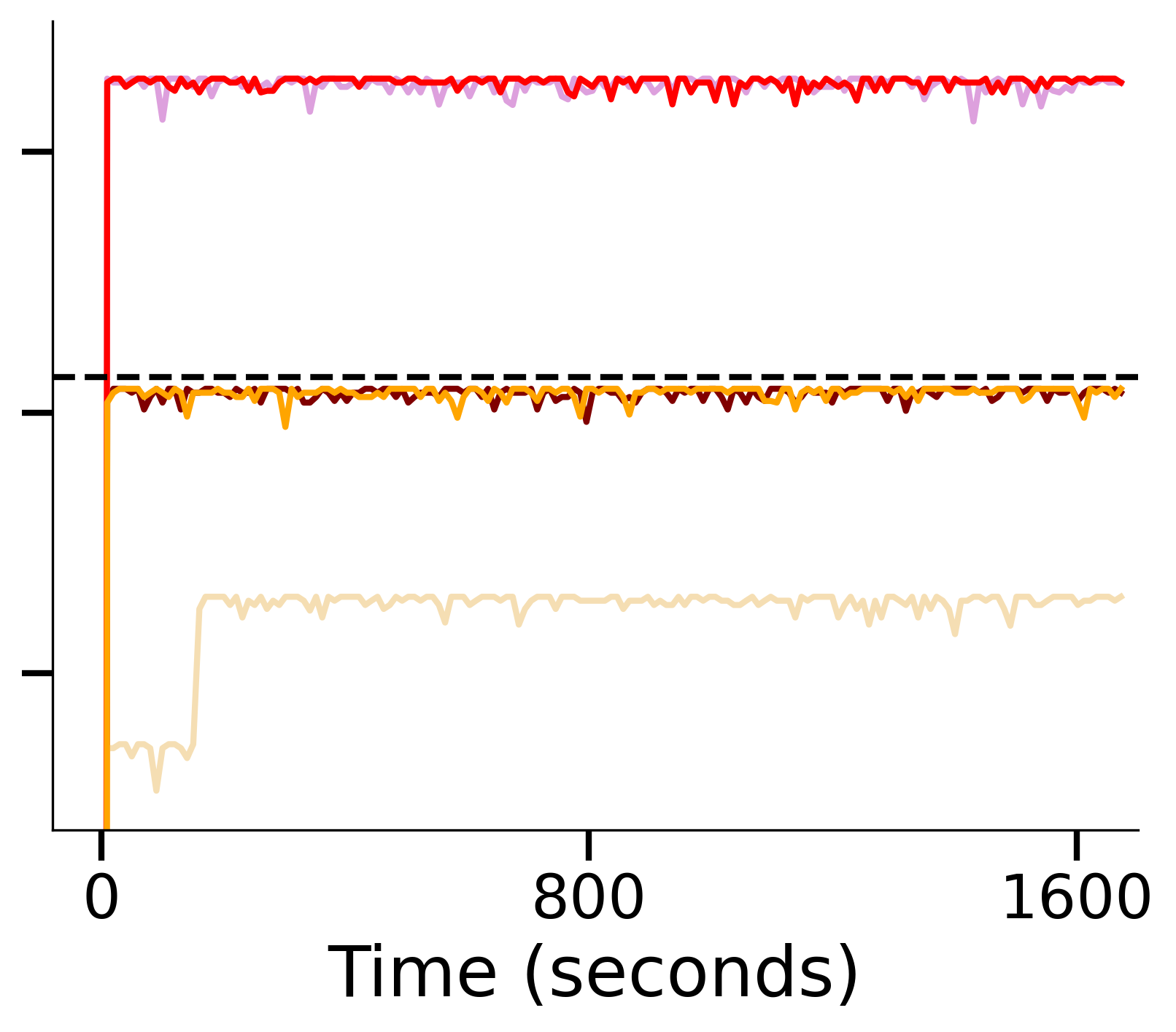} \\
\addlinespace[4pt]
\includegraphics[height=0.23\textwidth]{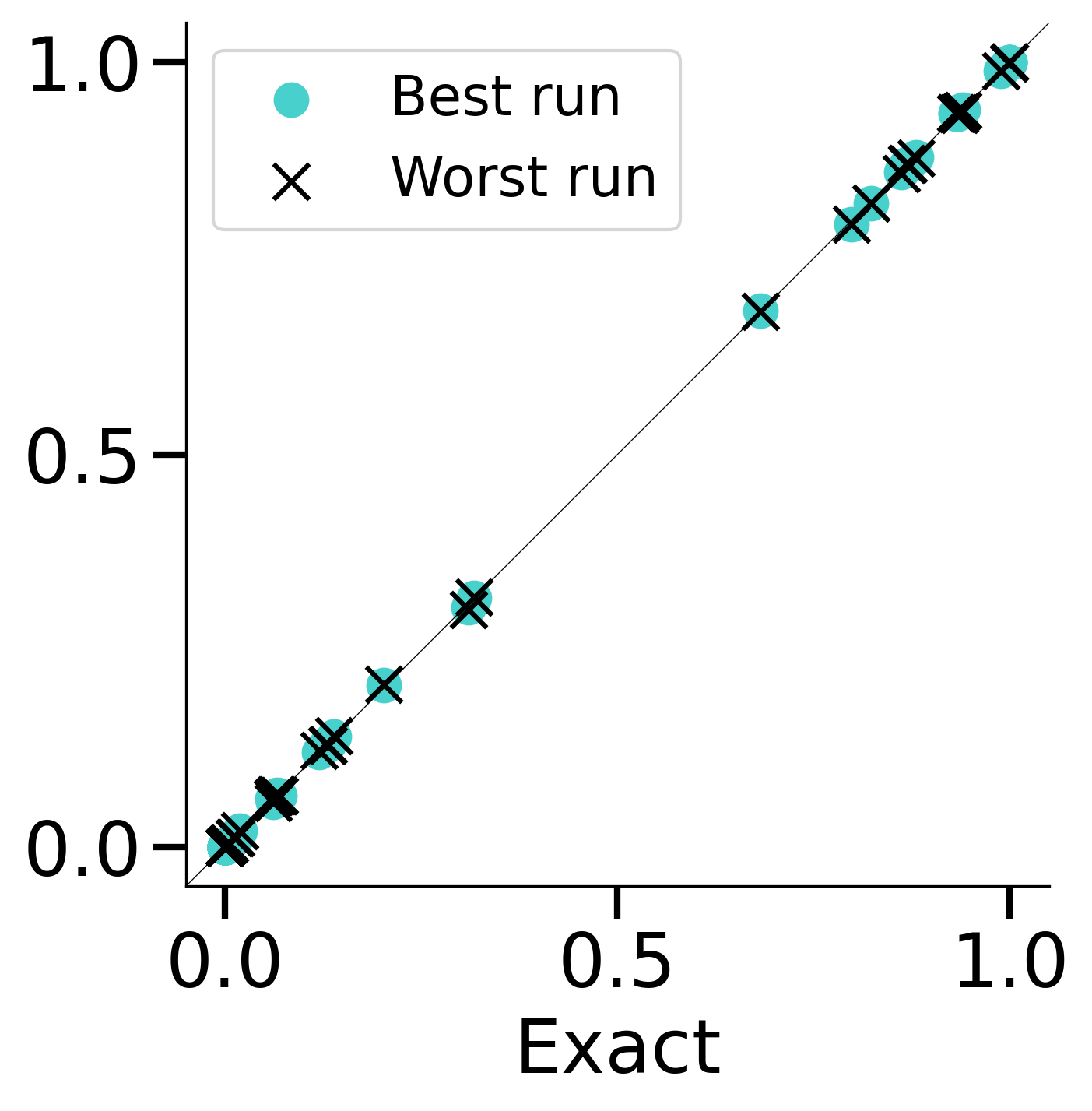} & 
\includegraphics[height=0.23\textwidth]{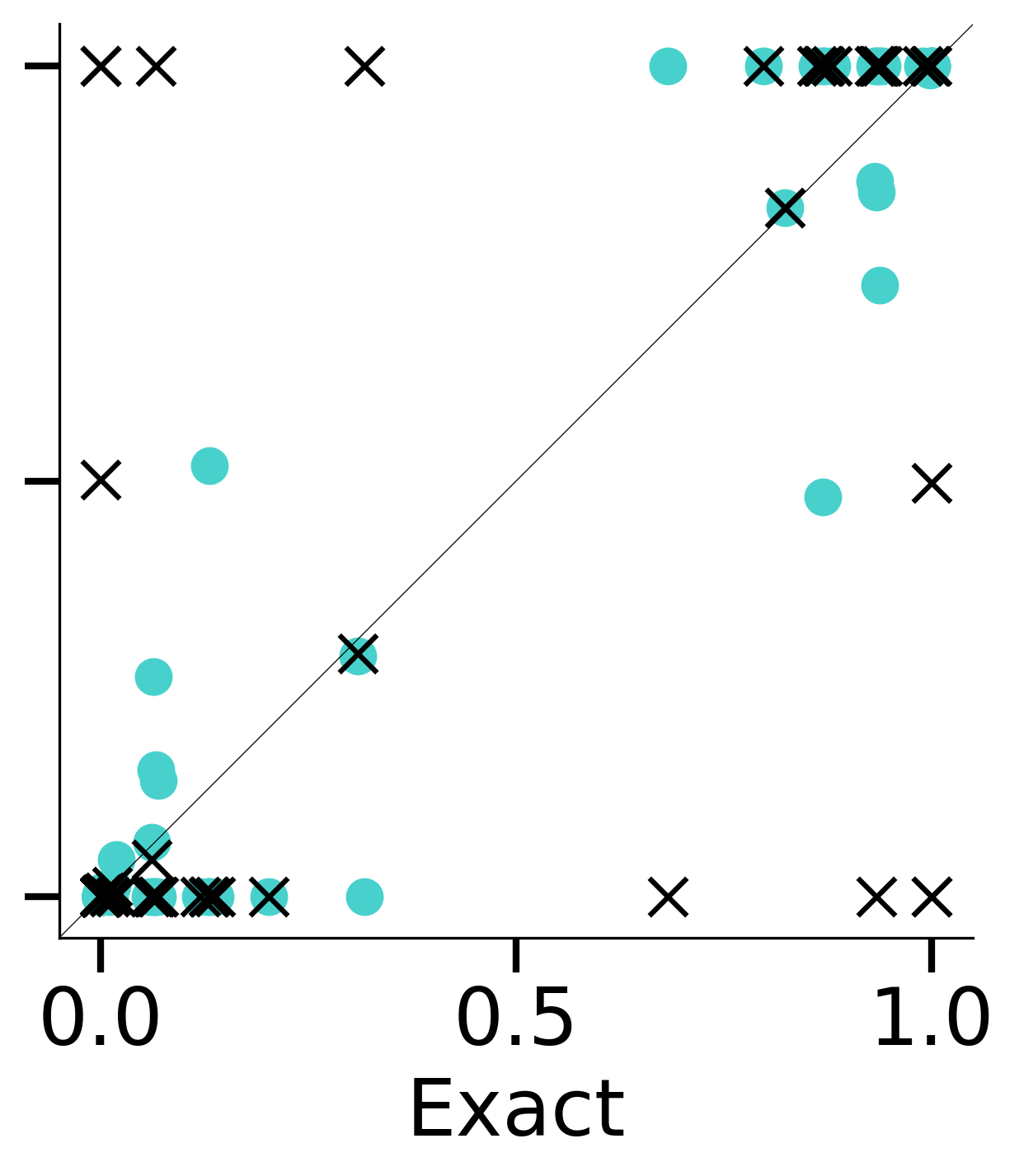} \\ 
\end{tabular}
}
\caption{\label{fig:tracechild}
Performance of DAG samplers on a \child{} dataset of size 1,{}000.
\emph{Top:} The log posterior probability of the sampled DAG in five independent runs. The posterior probability of the data-generating DAG $G^*$ is marked by a horizontal black line.
\emph{Bottom:} The estimated arc posterior probabilities at the end of the runs, against the exact values. For each node pair, the best and the worst estimate over the five runs are shown.
}
\end{center}
\end{figure*}

\begin{figure*}[t!]
\begin{center}
{\small
\begin{tabular}{cc}
\multicolumn{2}{c}{\zoo{} ($n = 17$)} \\
\midrule
\addlinespace[4pt]
\gibby{} & \bidag{} \\
\hspace{-10pt}
\includegraphics[height=0.23\textwidth]{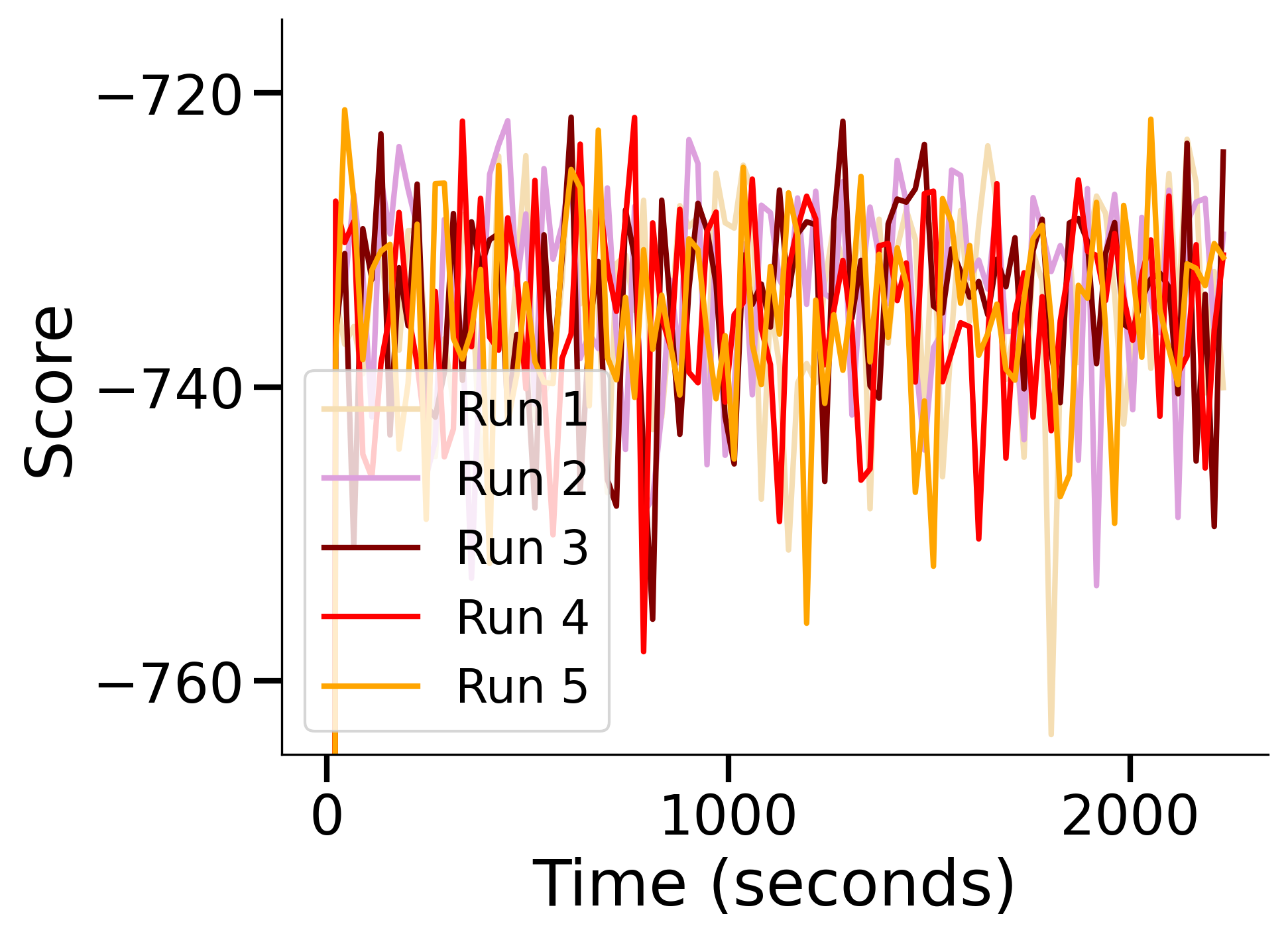} &  
\includegraphics[height=0.23\textwidth]{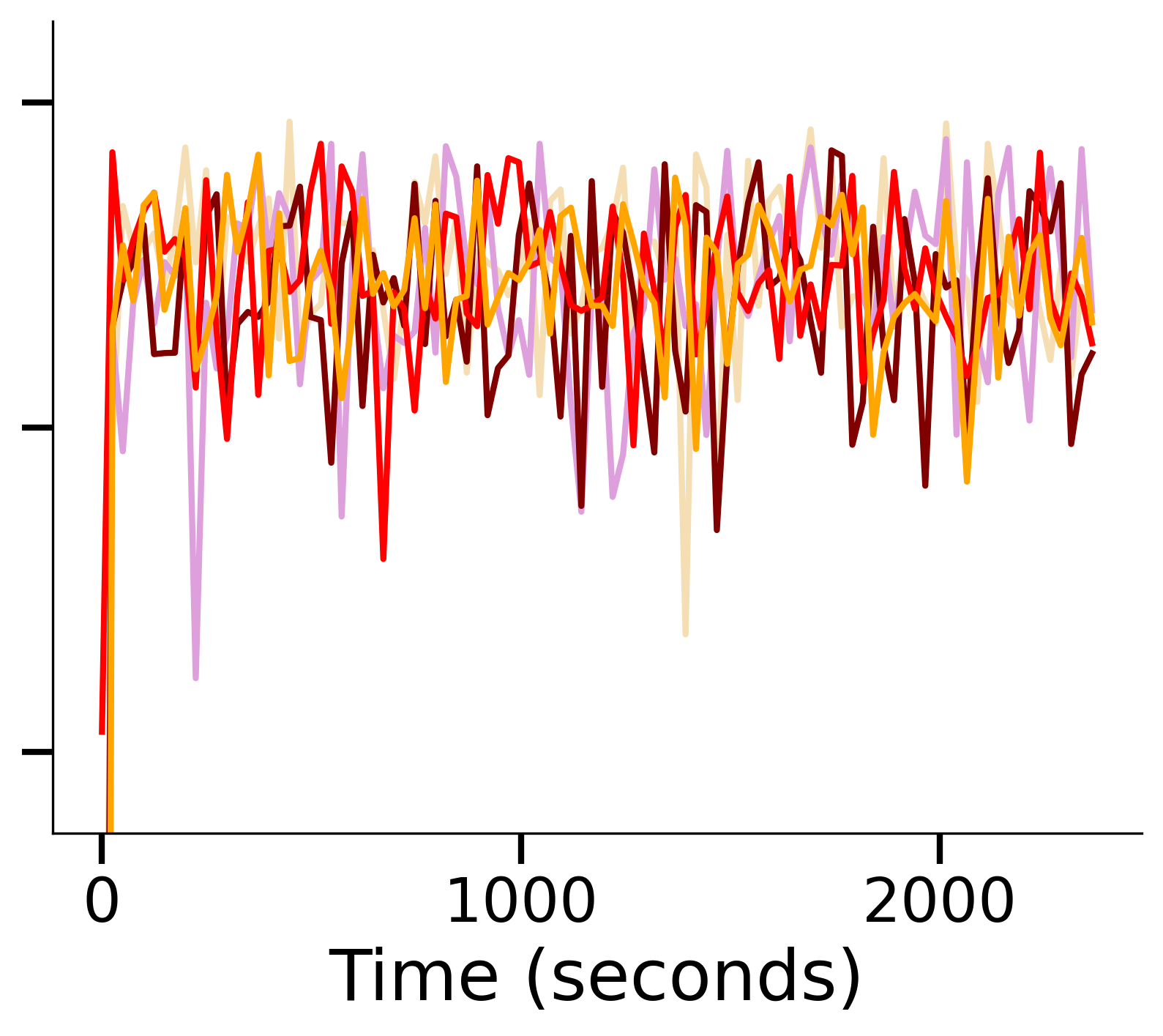} \\
\addlinespace[4pt]
\includegraphics[height=0.23\textwidth]{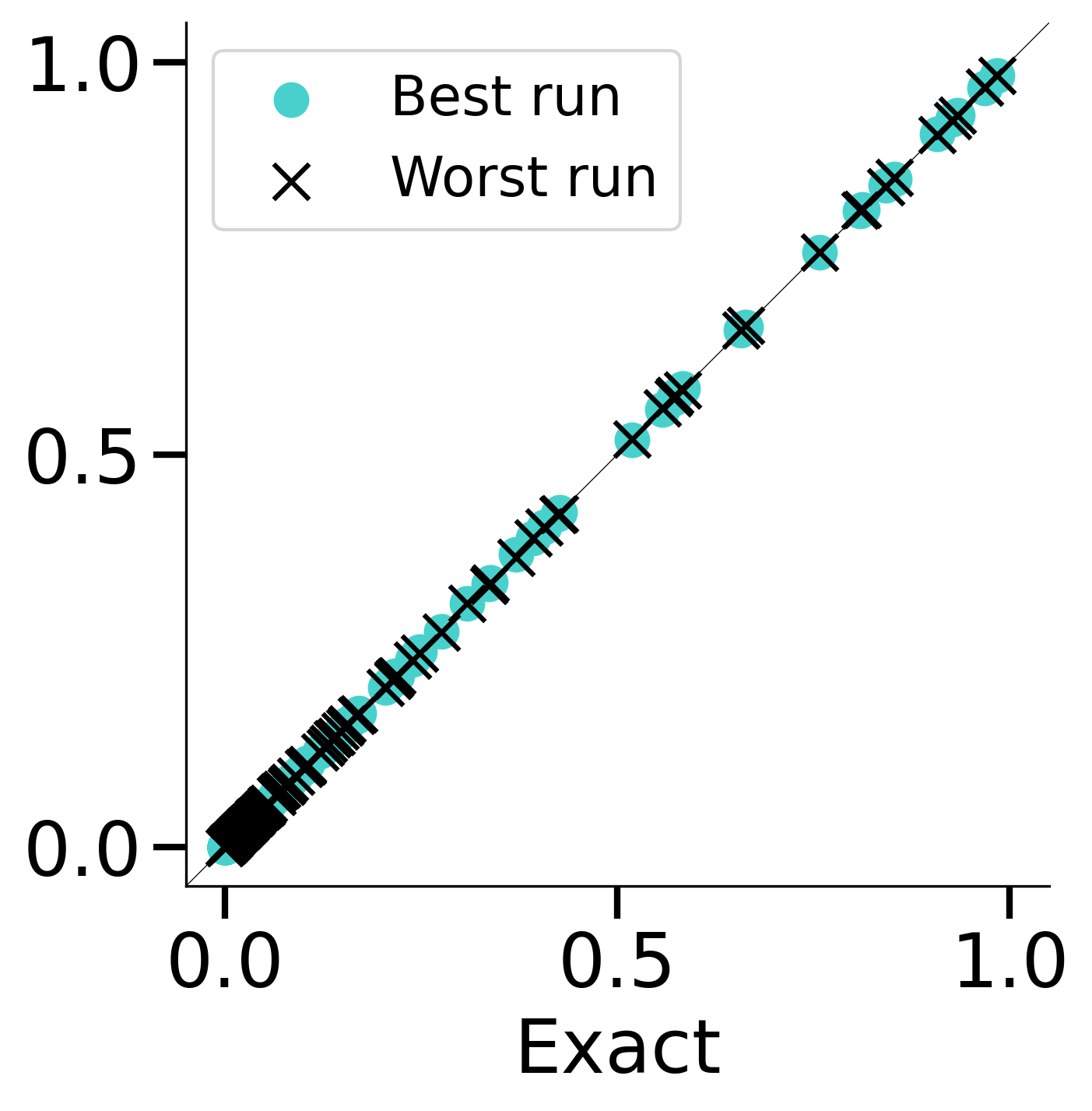} & 
\includegraphics[height=0.23\textwidth]{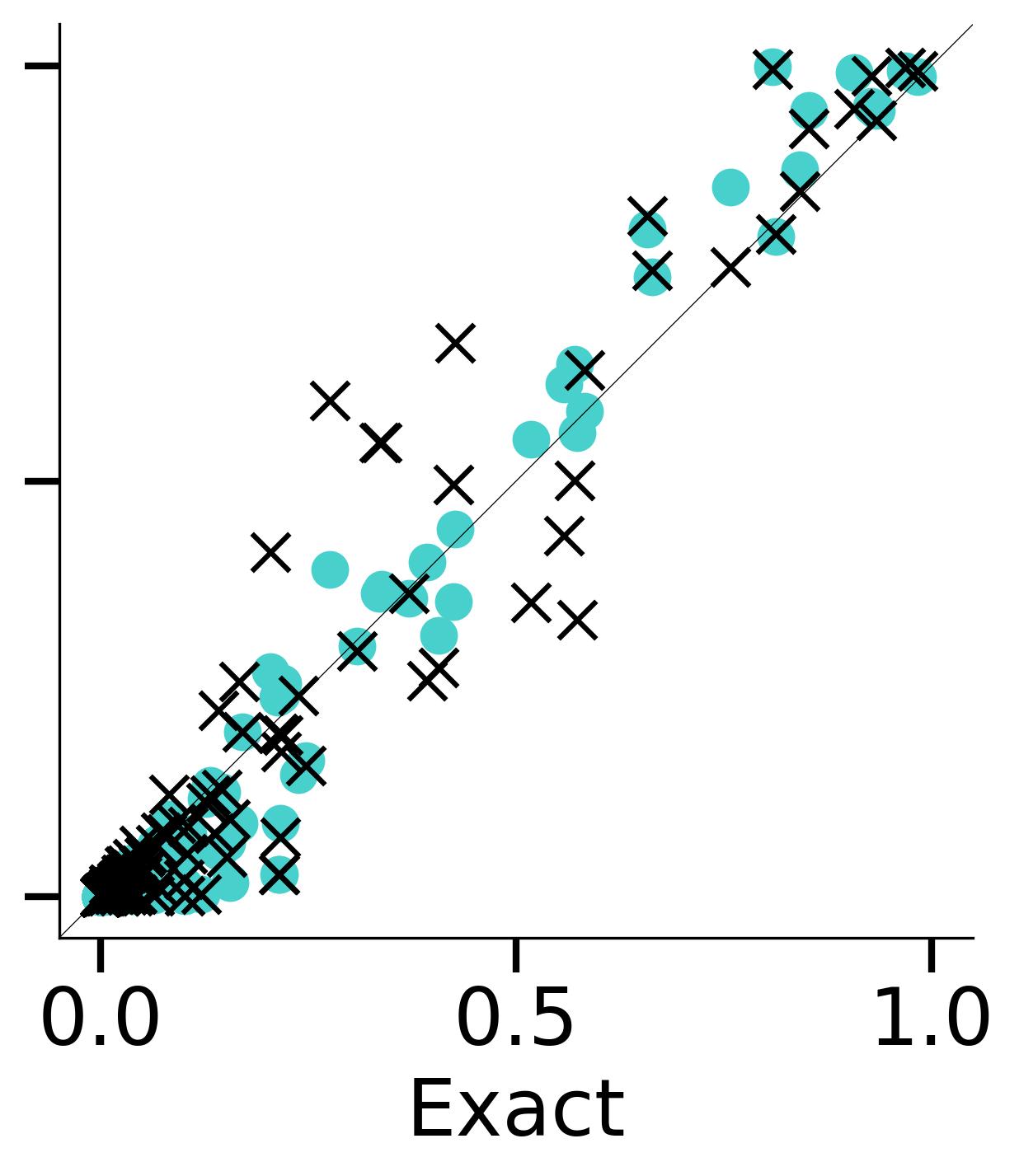} \\ 
\end{tabular}
}
\caption{\label{fig:tracezoo}
Performance of DAG samplers on the \zoo{} dataset (size 101).
\emph{Top:} The log posterior probability of the sampled DAG in five independent runs. 
\emph{Bottom:} The estimated arc posterior probabilities at the end of the runs, against the exact values. For each node pair, the best and the worst estimate over the five runs are shown.
}
\end{center}
\end{figure*}

Figures~\ref{fig:traceasia}--\ref{fig:tracezoo} compare the three DAG samplers on the datasets with 20 or fewer variables. When all the five independent runs of a sampler converge to about the same score range, then the estimates for the arc posterior probabilities also concentrate around the exact probability. The shown traceplots and correlation plots partly explain the MAD curves of \gibby{} and \bidag{} shown in Fig.~\ref{fig:small}. We managed to run \daggflow{} successfully for \asia{} and \sachs{}, for which it used about 30 minutes in the learning phase, after which it relatively quickly produced 1,{}000 independent DAG samples. We observe that the best runs of \daggflow{} give good estimates for \asia{} but not for \sachs{}.

Figure~\ref{fig:largecorr} shows the correlation of arc posterior probability estimates of different runs of the samplers on dataset generated from the larger benchmark BNs. The results suggest that \gibby{} is able to reliably produce accurate estimates. 

\begin{figure*}[t!]
\begin{center}
{\small
\begin{tabular}{cc}
\multicolumn{2}{c}{\hailfinder{} ($n = 56$)} \\
\midrule
\addlinespace[4pt]
\gibby{} & \hspace{32pt} \bidag{} \\
\includegraphics[height=0.23\textwidth]{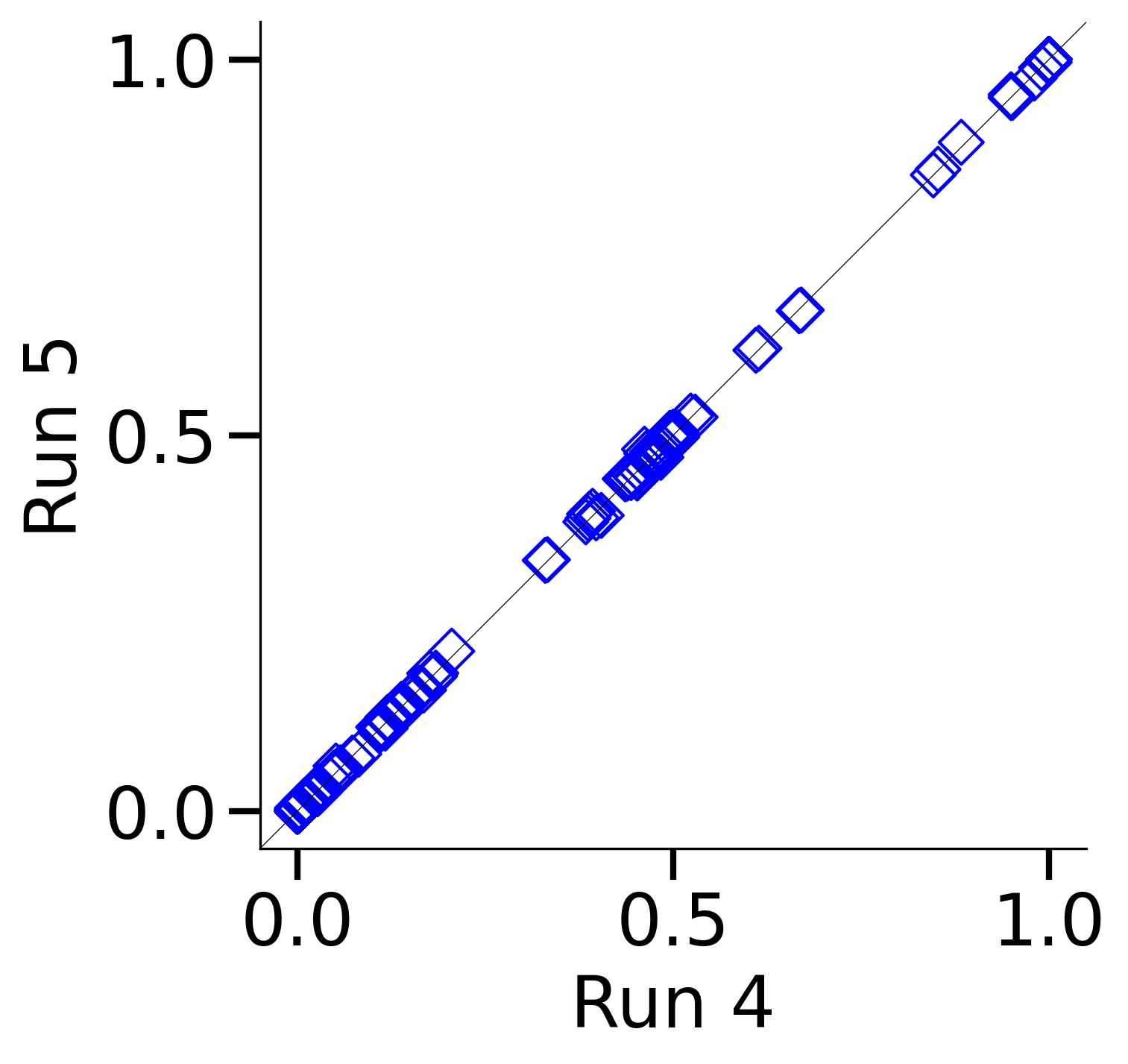} & 
\hspace{16pt}
\includegraphics[height=0.23\textwidth]{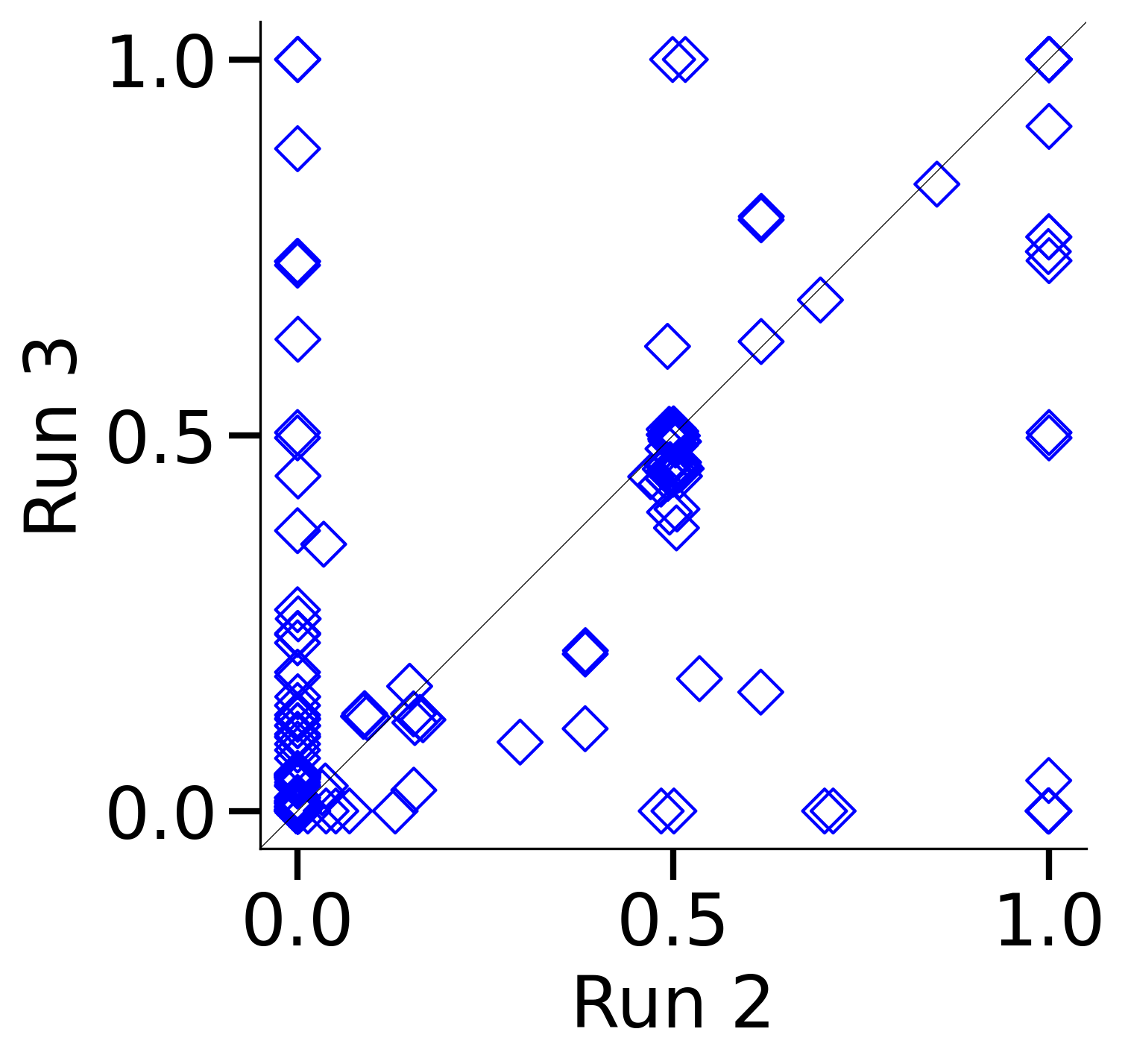} \\ 
\end{tabular}

\begin{tabular}{cc}
\multicolumn{2}{c}{\andes{} ($n = 233$)} \\
\midrule
\addlinespace[4pt]
\gibby{} & \hspace{32pt} \bidag{} \\
\includegraphics[height=0.23\textwidth]{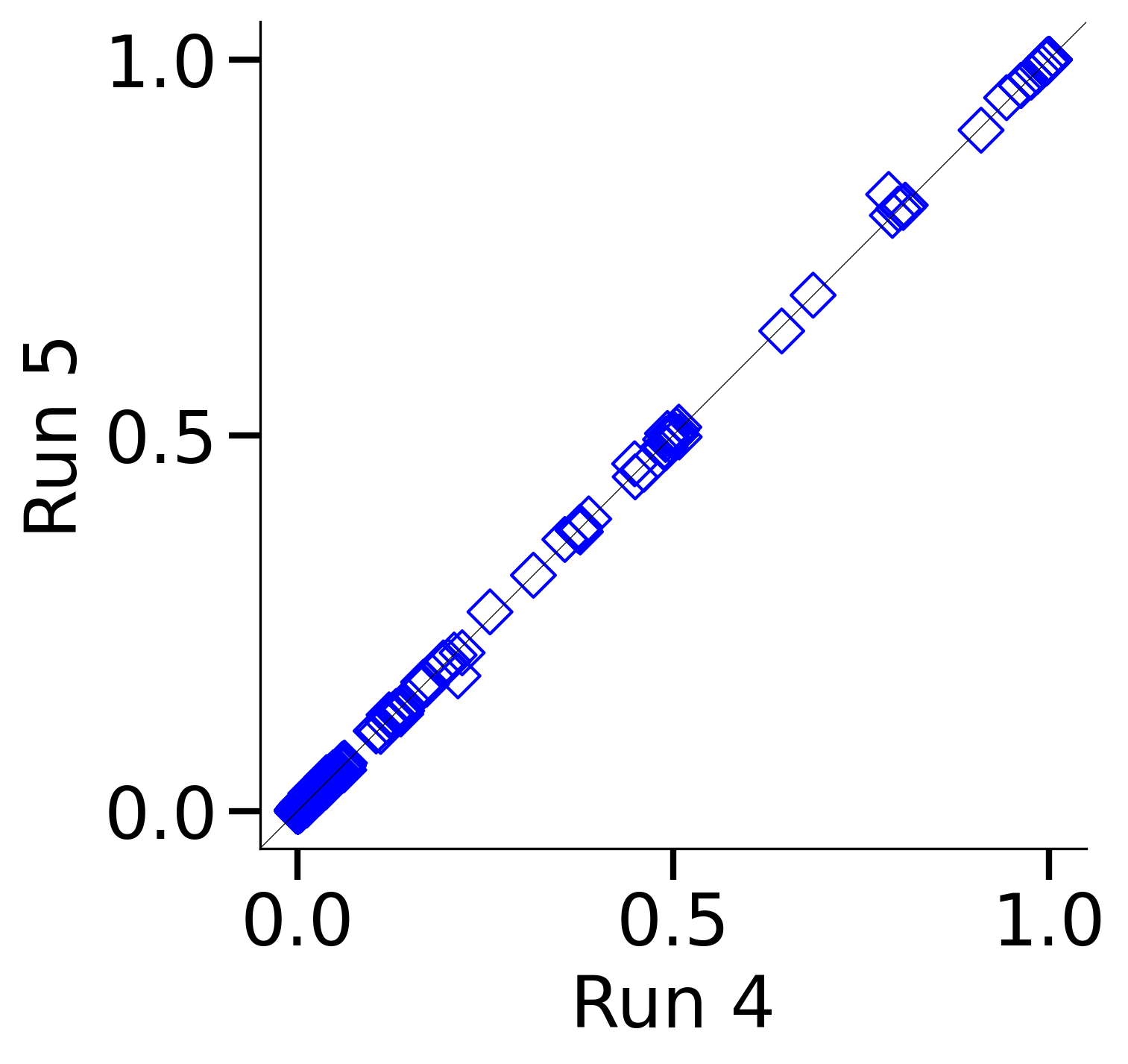} & 
\hspace{18pt}
\includegraphics[height=0.23\textwidth]{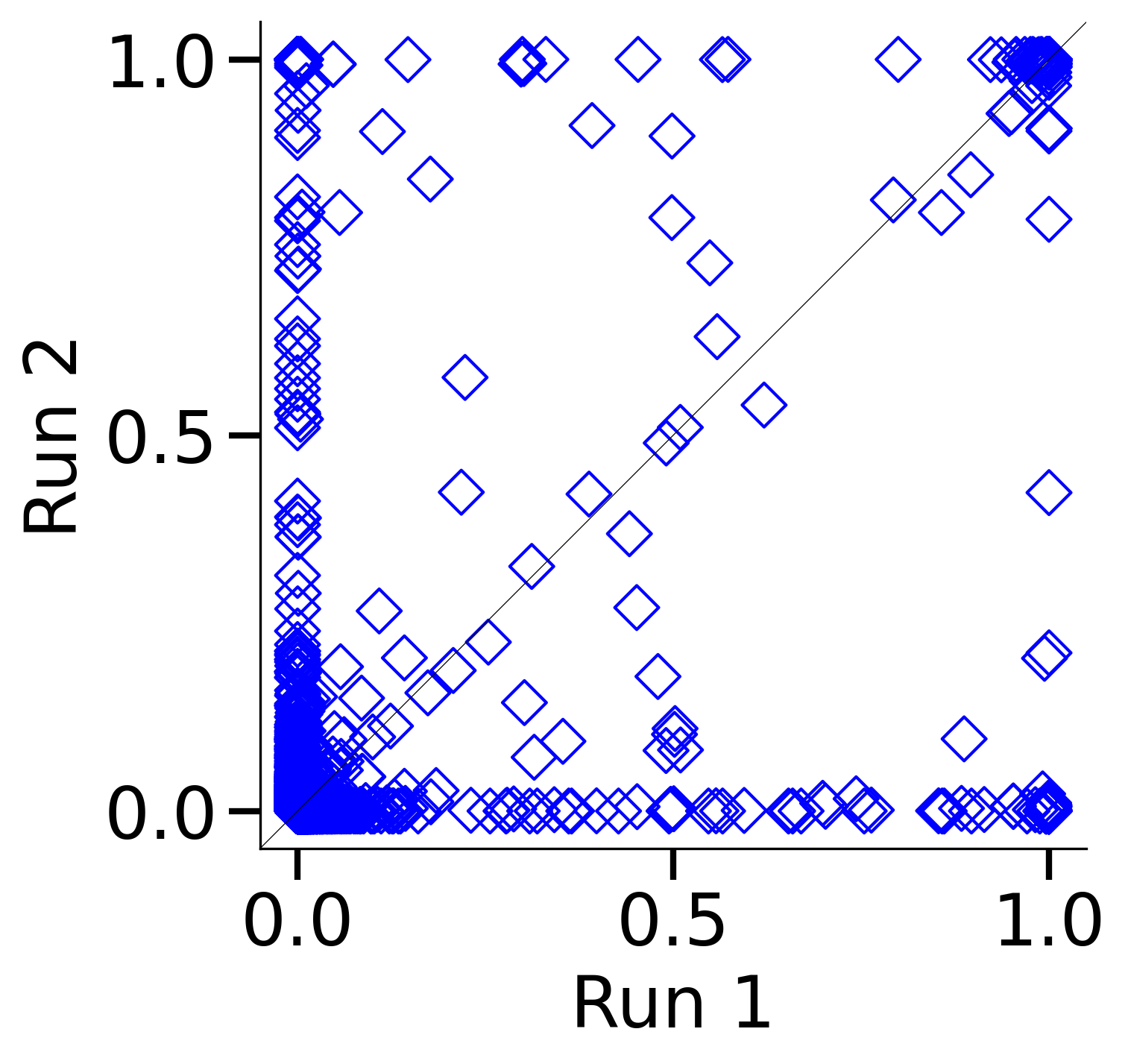} \\ 
\end{tabular}

}
\caption{\label{fig:largecorr}
Performance of DAG samplers on \hailfinder{} and \andes{} datasets of size 1,{}000.
For each sampler and dataset, shown are the estimated arc posterior probabilities for the ``worst'' pair out of five independent runs. 
The shown pair was selected to maximize the absolute difference of the estimates over all possible arcs. 
}
\end{center}
\end{figure*}

\end{document}